\documentclass{article} 
\usepackage{iclr2027_conference,times}

\usepackage{amsmath,amsfonts,bm}

\def\eqref#1{equation~\ref{#1}}

\def\1{\bm{1}}

\DeclareMathAlphabet{\mathsfit}{\encodingdefault}{\sfdefault}{m}{sl}
\SetMathAlphabet{\mathsfit}{bold}{\encodingdefault}{\sfdefault}{bx}{n}

\usepackage{hyperref}
\usepackage{url}
\usepackage{xcolor}
\definecolor{linkblue}{HTML}{3674F6}
\usepackage{colortbl} 
\usepackage{booktabs} 
\usepackage{graphicx}
\usepackage{amsmath,amssymb,amsthm}
\usepackage{multirow}
\usepackage{wrapfig}
\usepackage{etoc}
\usepackage{needspace}
\usepackage{placeins}
\newtheorem{lemma}{Lemma}
\newtheorem{proposition}{Proposition}
\newtheorem{corollary}{Corollary}
\theoremstyle{remark}
\newtheorem{remark}{Remark}
\definecolor{myblue}{RGB}{135,206,235}

\title{GaugeVLM: Structuring Spatial Supervision with Measured Geometric Interventions}

\author{Hongbo Wang$^{1,2}$\quad Zihan Lin$^{1,3}$\quad Wenkui Yang$^{1,3}$\quad Shiran Ge$^{1,2,4}$\quad Yuang Ai$^{5}$ \\
\textbf{Jie Cao}$^{1,2}$\quad \textbf{Huaibo Huang}$^{1,2}$\quad \textbf{Ran He}$^{1,2,3}\thanks{Corresponding author.}$
\\
\\
$^1$ NLPR \& MAIS, Institute of Automation, Chinese Academy of Sciences \\
$^2$ School of Artificial Intelligence, University of Chinese Academy of Sciences \\
$^3$ School of Advanced Interdisciplinary Sciences, University of Chinese Academy of Sciences \\
$^4$ National University of Singapore \\
$^5$ The Chinese University of Hong Kong
\\
\\
\textbf{\textcolor{linkblue}{\url{https://wafer-bob.github.io/GaugeVLM/}}}
}

\newcommand{\model}{GaugeVLM}
\newcommand{\obj}{GaugeDPO}
\newcommand{\data}{Gauge-50K}
\newcommand{\bench}{Constancy-Bench}

\newcommand{\method}{\model}
\newcommand{\dring}{d_{\mathrm{ring}}}
\newcommand{\drel}{d_{\mathrm{rel}}}

\providecommand{\gain}[1]{\textcolor{teal!70!black}{\scriptsize$+#1$}}
\providecommand{\drop}[1]{\textcolor{red!75!black}{\scriptsize$-#1$}}

\iclrfinalcopy

\begin{document}
\etocdepthtag.toc{main}

\maketitle

\begin{figure}[h]
    \vspace{-5pt}
    \centering
    \includegraphics[width=.95\linewidth]{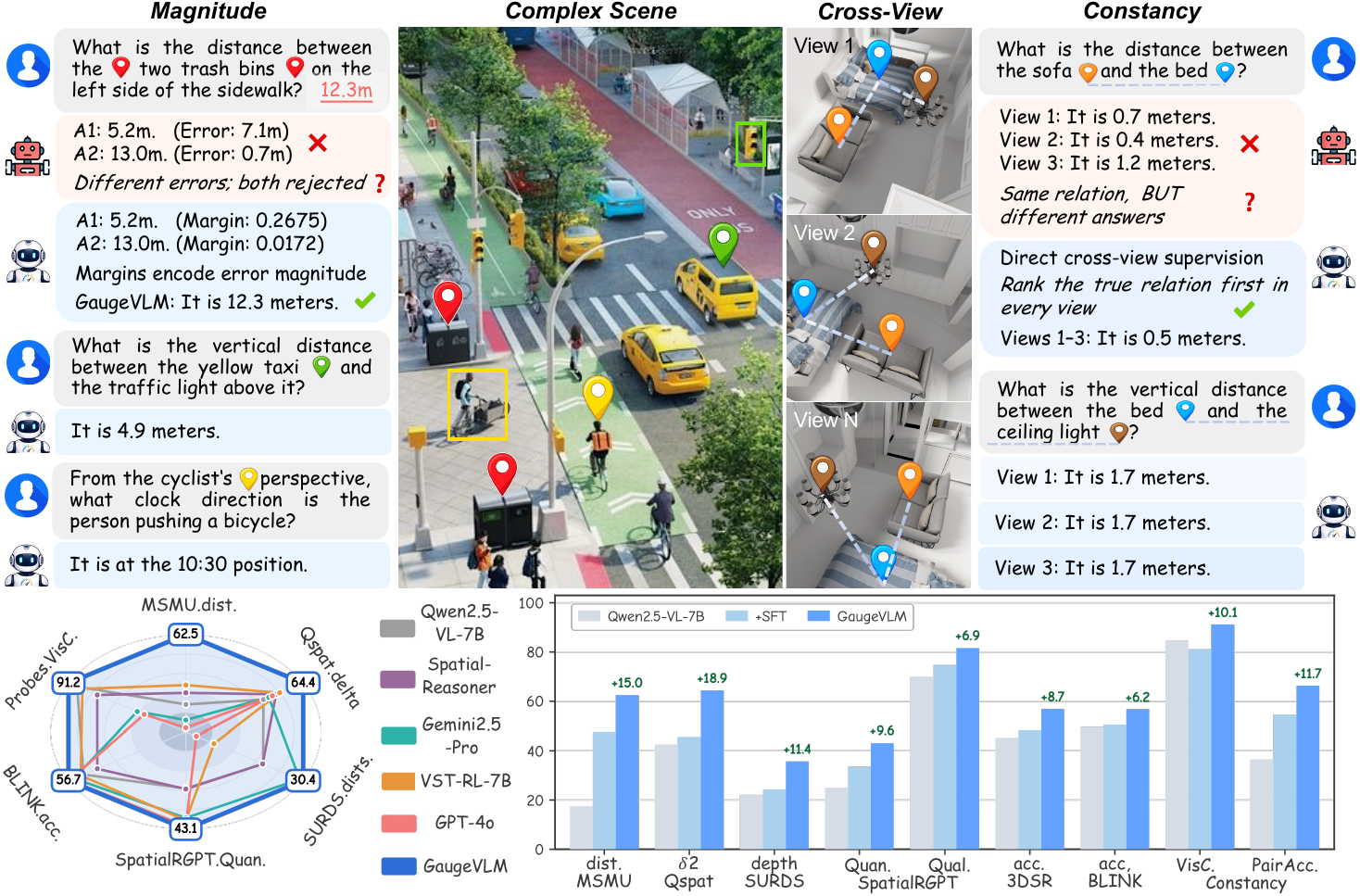}
    \label{fig:placeholder}
\end{figure}

\begin{abstract}
Vision-language models (VLMs) can contradict themselves across views of the same spatial relation and fail to respond when that relation changes. Addressing these failures requires supervision that captures error magnitude and geometric dependencies across observations, both of which remain implicit in training on individual answers or ordinal preferences. Therefore, we introduce \textbf{GaugeVLM}, which makes this structure explicit through controlled object and camera interventions in explicit 3D scenes, producing linked observations with measured differences between spatial relations and shared truths across views. To translate this structure into learning signals, its core objective, \textbf{GaugeDPO}, converts measured errors into preference margins, directly supervises correct canonical rankings across views, and links intervention-induced answer-odds contrasts to measured relation changes with view-specific scales. Our analysis bounds canonical prediction error and establishes that the cross-view and intervention constraints can be jointly satisfied. Empirically, GaugeVLM improves all 10 established spatial metrics over supervised fine-tuning across three VLM backbones, with the main 7B model gaining 15.0 and 18.9 percentage points on MSMU distance and QSpatial+, respectively. These gains also extend to autonomous driving and embodied reasoning, demonstrating the robust generalization across domains.
\end{abstract}

\section{Introduction}

Despite rapid advances in vision-language models (VLMs)~\citep{openaiGPT4TechnicalReport2023, chenInternVLScalingVision2024, baiQwen25VLTechnicalReport2025}, spatial reasoning remains brittle. Even after large-scale supervised fine-tuning (SFT) on spatial question answering~\citep{chenSpatialVLMEndowingVisionLanguage2024, chenSDVLMSpatialMeasuring2025, ogeziSpaREEnhancingSpatial2025, heluScalingSpatialReasoning2025}, VLMs still struggle with distances, directions, and relative relations~\citep{liaoReasoningPathsReference2024, ma3DSRBenchComprehensive3D2024, fuBLINKMultimodalLarge2024}, especially when relations must hold across viewpoints~\citep{liViewSpatialBenchEvaluatingMultiperspective2025, yangThinkingSpaceHow2025, yehSeeingAnotherPerspective2025, wangMindCubeSpatialMental2025}. These failures motivate examining the supervision signal alongside data and model capacity. Supervision based only on target answers tells the model what the answer is, but not how far an error is from the truth or whether the same relation should remain consistent across viewpoints.

Preference optimization offers a way to use incorrect answers as supervision, but the ordinal labels used by standard Direct Preference Optimization (DPO)~\citep{rafailovDirectPreferenceOptimization2024} do not quantify geometric error. For the same true distance, predictions off by ten centimeters and ten meters receive the same rejected label despite their different error magnitudes. This motivates an instance level offset that explicitly reflects how far a rejected answer deviates from the truth. Likewise, since the same allocentric error has the same geometric magnitude across camera views, its assigned offset should also remain unchanged. The central question is therefore how to derive offsets that capture error magnitude while remaining invariant to viewpoint.

Geometry provides the missing scale for spatial preferences. \textbf{\obj{}} grounds each preference offset in measured error, making supervision sensitive to error magnitude while assigning the same offset to the same allocentric error across views. The measurement comes from the construction itself. \textbf{\model{}} perturbs known relations in explicit 3D scenes and verbalizes the resulting alternatives as rejected answers, making their geometric deviations directly available without an intermediate reward estimate. These deviations should depend on geometry regardless of angular origin or distance units. We therefore build a bounded metric from circular direction differences and distance ratios, with explicit saturation. Used as a preference offset, this metric translates the severity absent from ordinal labels into the separation encouraged during training. To understand how it does so, we analyze the offset's local gradient effect and characterize the allocation of preference separation under an equal weight budget, then examine the learned mean gaps empirically.

Beyond measuring individual errors, we develop a 3D scene engine for \model{} to connect observations through controlled changes in the scene. Object interventions produce measured changes in the true relation, while camera interventions preserve it under a fixed object configuration and anchor frame. These linked observations specify when predictions should change and when they should remain constant. Accordingly, direct supervision enforces correct canonical rankings across views, while an intervention profile links answer odds contrasts to measured relation changes without requiring equal confidence across views. The same distinction guides evaluation. For views sharing a fixed relation, we show that half the largest disagreement in a common frame lower bounds the largest prediction error. Object interventions instead require checking predictions against their changed truths. The engine supplies \textbf{\data{}} for training and a probe within \textbf{\bench{}}.

Across three backbone families, \model{} improves all ten established spatial metrics over SFT. The main 7B model gains $15.0$ and $18.9$ percentage points on MSMU distance and QSpatial+, respectively, and leads the listed prior open spatial VLMs on nine of ten metrics under the evaluation. Gains extend to driving scenarios, including an $11.4$ percentage point improvement on SURDS depth. Matched ablations support the value of measured margins over constant, shuffled, and estimated reward margins. With only $25\%$ of the preference data, \obj{} surpasses DPO trained on the full preference dataset across all five evaluated metrics. Further evaluations show improved correctness across views and generalization to unseen intervention magnitudes and combinations.
\vspace{-5pt}
\section{Spatial Preference Needs a Geometric Margin}
\label{sec:background}

\subsection{Geometry-Grounded Offsets}
\label{sec:related}
DPO learns ordinal preferences, SimPO, IPO, AlphaDPO and ODPO
regulate gaps or offsets~\citep{rafailovDirectPreferenceOptimization2024,mengSimPOSimplePreference2024,azarGeneralTheoreticalParadigm2023a,wuAlphaDPOAdaptiveReward2025,aminiDirectPreferenceOptimization2024};
mDPO adds image comparisons~\citep{wangMDPOConditionalPreference2024}.
Spatial methods use QA, grounding and graded
preferences~\citep{chenSpatialVLMEndowingVisionLanguage2024,chenSDVLMSpatialMeasuring2025,ogeziSpaREEnhancingSpatial2025,heluScalingSpatialReasoning2025,chengSpatialRGPTGroundedSpatial,shenFineGrainedPreferenceOptimization2026},
transformed-input coupling or smooth numerical rewards~\citep{wangSVQAR1ReinforcingSpatial2025,jiaoSmoothOperatorSmooth2026}.
Ordinal labels do not quantify errors. We map controlled 3D perturbations to offsets with a bounded, unit-invariant metric. Moreover, cross-view and counterfactual
studies~\citep{bhatConsistentWrongEvidence2026,vellamchetiCVTBenchCounterfactualViewpoint2026,huangConsiSpaceLearningGeometric2026,yuMindEditBenchBenchmarkingObjectLevel2026}
motivate separate tests of viewpoint constancy and correctness
before and after target relocation.

\subsection{Formal Setup and Design Requirements}
\label{sec:prelim}
\label{sec:requirements}
Quantization maps a scene's physical relation in a fixed anchor frame to
$r^\star=Q(r_{\rm phys}^\star)\in\mathcal R_Q\subset\mathcal R$.
Camera $v$ supplies image--prompt input $x_v$. On $\mathcal R_Q$, verbalizer $\Phi$ and
parser $\Pi$ satisfy $\Pi(\Phi(r,v))=r$; invalid answers map to $\bot$.
Response-side pairs compare true and distinct rejected labels under one
input; camera-relative binary probe labels are separate.
For chosen--rejected triples $(x,y_w,y_l)\sim\mathcal D$ and $\beta>0$,
define the reference-relative score
$\rho_\theta(x,y)=\beta\log[\pi_\theta(y\mid x)/\pi_{\rm ref}(y\mid x)]$
with positive reference probabilities. With $\sigma(z)=(1+e^{-z})^{-1}$,
\begin{equation}
\mathcal L_{\rm Margin\text{-}DPO}
=-\mathbb E_{\mathcal D}\log\sigma\!\left(
\rho_\theta(x,y_w)-\rho_\theta(x,y_l)-\gamma\right),
\label{eq:margin-dpo}
\end{equation}
The prescribed offset $\gamma\ge0$ shifts the gap required at a fixed
loss level~\citep{aminiDirectPreferenceOptimization2024,wuAlphaDPOAdaptiveReward2025},
with $\gamma=0$ recovering DPO.

For spatial relations, the offset should reflect geometric error, while
the policy's predictions should preserve the true relation across cameras.
These impose distinct requirements on margins and predictions. Let
$\mathcal E:\mathcal R\times\mathcal R\to\mathbb R_{\ge0}^{k}$
collect deviations, with $\mathcal E(r,r^\star)=0$ iff $r=r^\star$.
\begin{description}
\item[(R1) Magnitude.]
The margin $\tilde\gamma:\mathcal R^2\times V\to\mathbb R_{\ge0}$
factors as $\tilde\gamma(r^-,r^\star,v)=\gamma(\mathcal E(r^-,r^\star))$,
where $\gamma$ is componentwise non-decreasing and strictly increasing
before saturation, with $\gamma(e)=0$ iff $e=0$. Thus each displacement
receives the same assigned margin across cameras.
\item[(R2) Constancy.]
With the object configuration and anchor frame fixed, an ideal policy
satisfies $\Pi(y_v)=r^\star$ for every identifiable view $v\in V$.
Correctness implies agreement; agreement alone permits consistent error.
Matching training statistics does not establish this property.
\end{description}
Camera invariance alone is insufficient: a constant margin fails (R1).
We characterize coordinate-invariant deviations and use a metric to assign
offsets, supplying geometric information absent from deterministic
ordinal labels.

\section{\obj{}: Learning from Geometric Interventions}
\label{sec:method}

\begin{figure}[t]
\vspace{-5pt}
    \centering
    \includegraphics[width=0.98\linewidth]{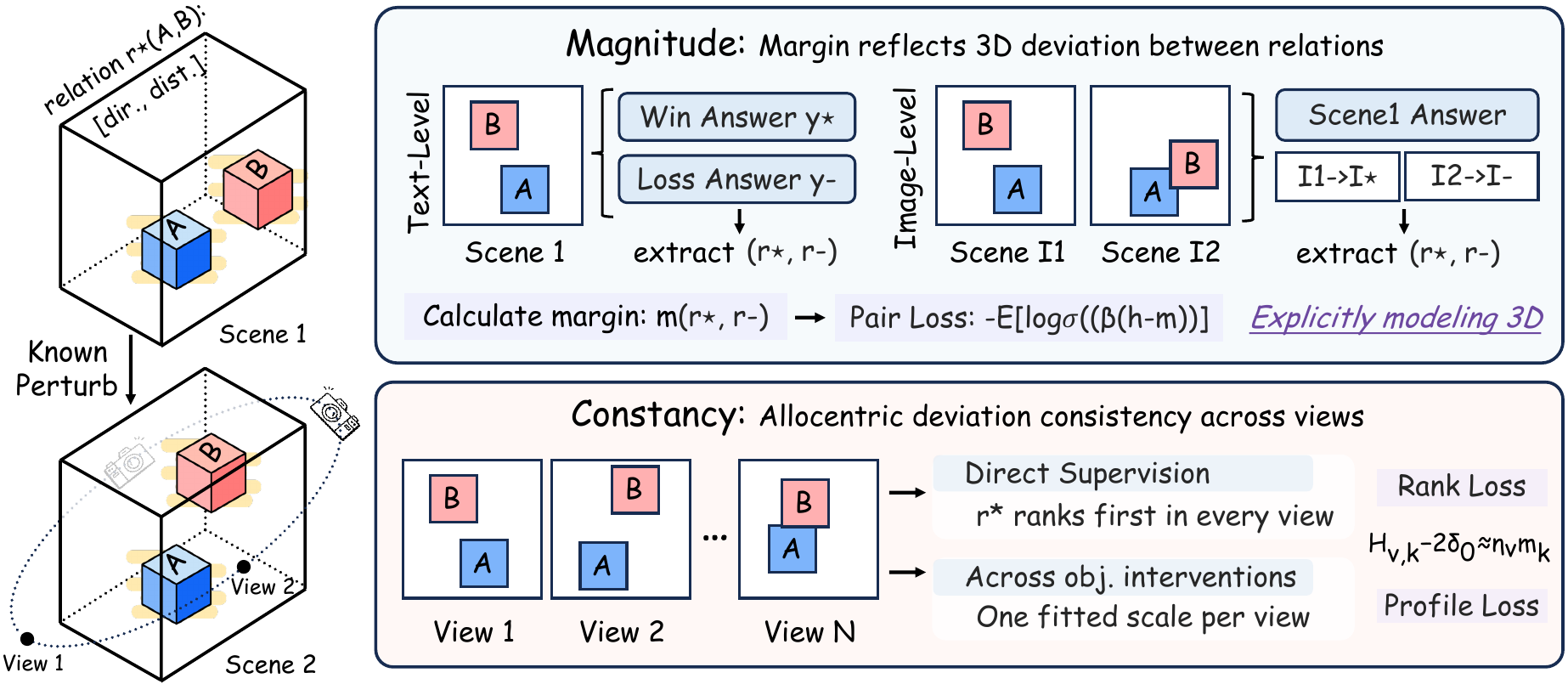}
    \caption{\textbf{\obj{} turns measured 3D interventions into supervision.} \emph{Left,} controlled perturbations define the margin $m(r^\star,r^-)$. \emph{Top,} it shifts response- and image-side pair losses. \emph{Bottom,} direct supervision and intervention profiles enforce correct, view-consistent predictions.}
    \label{fig:pipeline}
\end{figure}

Object interventions alter relations and camera changes preserve their truth.
\obj{} turns this geometry into preference and prediction margins,
then structures canonical answer support across interventions.

\subsection{From Geometry to Matched Preferences}
\label{sec:margin}
\label{sec:constancy}

\paragraph{A shared geometric metric.}
Let $r=(c,d)\in\mathcal R:=\mathcal C_{12}\times\mathbb R_{>0}$, with $\mathcal C_{12}=\mathbb Z/12\mathbb Z$. Clock hours advance clockwise about gravity from the anchor's canonical front at $12$ o'clock; $d$ is centroid distance in meters. Camera-independent errors should respect circular shifts and changes of units:
\begin{align}
D_c(c+k,c^\star+k)=D_c(c,c^\star),\text{ } D_d(\mu d,\mu d^\star)=D_d(d,d^\star),\quad\text{ } k\in\mathcal C_{12}, \text{ }\mu>0,
\label{eq:dist-inv}
\end{align}
We therefore use clock difference and distance ratio, with shortest-arc and capped log-distance:
\begin{gather}
\dring(c,c^\star)=\tfrac16\lvert c-c^\star\rvert_{12},
\quad
\drel(d,d^\star)=\tfrac1\kappa
\min\!\left(\left|\log(d/d^\star)\right|,\kappa\right),
\label{eq:ring}\\
m(r,r^\star)=
\alpha\,\dring(c,c^\star)
+(1-\alpha)\,\drel(d,d^\star),
\label{eq:margin}
\end{gather}
Here $\lvert\cdot\rvert_{12}\in\{0,\ldots,6\}$, $\kappa>0$ and
$\alpha\in(0,1)$. The metric $m\in[0,1]$ satisfies (R1) with
$\mathcal E=(\dring,\drel)$,
$\gamma(e)=\beta[\alpha e_c+(1-\alpha)e_d]$ and $\widetilde\gamma=\beta m$.

\paragraph{Matched response and image changes.}
Let $\Pi$ extract the terminal relation from the multi-step response
$\Phi(r,v)$, with representable domain $\mathcal R_Q\subset\mathcal R$.
A displacement $\delta\in\mathcal C_{12}\times\mathbb R_{>0}$ maps
$r^\star=(c^\star,d^\star)\in\mathcal R_Q$ to
$r^-=(c^\star+\delta_c,d^\star\delta_d)\ne r^\star$ in $\mathcal R_Q$,
fixing geometric error. Response-side pairs compare
$\Phi(r^\star,v)$ and $\Phi(r^-,v)$ on one input; image-side pairs
change the scene to $r^-$, fixing the question and response
$\Phi(r^\star,v)$. Symmetry gives the same rejected error.
Margins measure terminal error; preference scores use full responses.
For fields $J\in{{c},{d},{c,d}}$ predeclared,
$\Pi_J$ outputs in $\mathcal R_J=\prod_{j\in J}\mathcal R_j$, where
$\mathcal R_c=\mathcal C_{12}$ and $\mathcal R_d=\mathbb R_{>0}$.
For $u,u^\star\in\mathcal R_J$, use
\begin{equation}
m_J(u,u^\star)=\sum_{j\in J}a_j D_j(u_j,u_j^\star),
\quad(a_c,a_d)=(\alpha,1-\alpha),\quad(D_c,D_d)=(\dring,\drel),
\label{eq:projected-margin}
\end{equation}
Each $m_J$ is a metric satisfying (R1) with unchanged weights;
$m_{\{c,d\}}=m$ and $\Pi_{\{c,d\}}=\Pi$. Tasks retain their response formats.

\begin{proposition}[Constancy certificate]
\label{prop:certificate}
Let $G=\{x_v\}_{v=1}^{V_G}$, $V_G\ge2$, contain camera views of a
fixed object configuration and anchor frame, sharing truth $r_G^\star$.
For responses $\{y_v\}$, let $V_{\rm ok}=\{v:\Pi(y_v)\ne\bot\}$
and $a_v=\Pi(y_v)$ on valid views. If $V_{\rm ok}\ne\emptyset$,
\begin{equation}
\max_{v\in V_{\rm ok}}m(a_v,r_G^\star)
\ge \tfrac12\max_{u,v\in V_{\rm ok}}m(a_u,a_v),
\label{eq:certificate}
\end{equation}
The factor $\tfrac12$ is optimal; agreement alone does not certify correctness.
\end{proposition}

\subsection{Learning Magnitude and Correctness Across Views}
\label{sec:objective}
\label{sec:constancy-term}

\begin{figure*}[t]
    \vspace{-10pt}
    \centering
    \includegraphics[width=\linewidth]{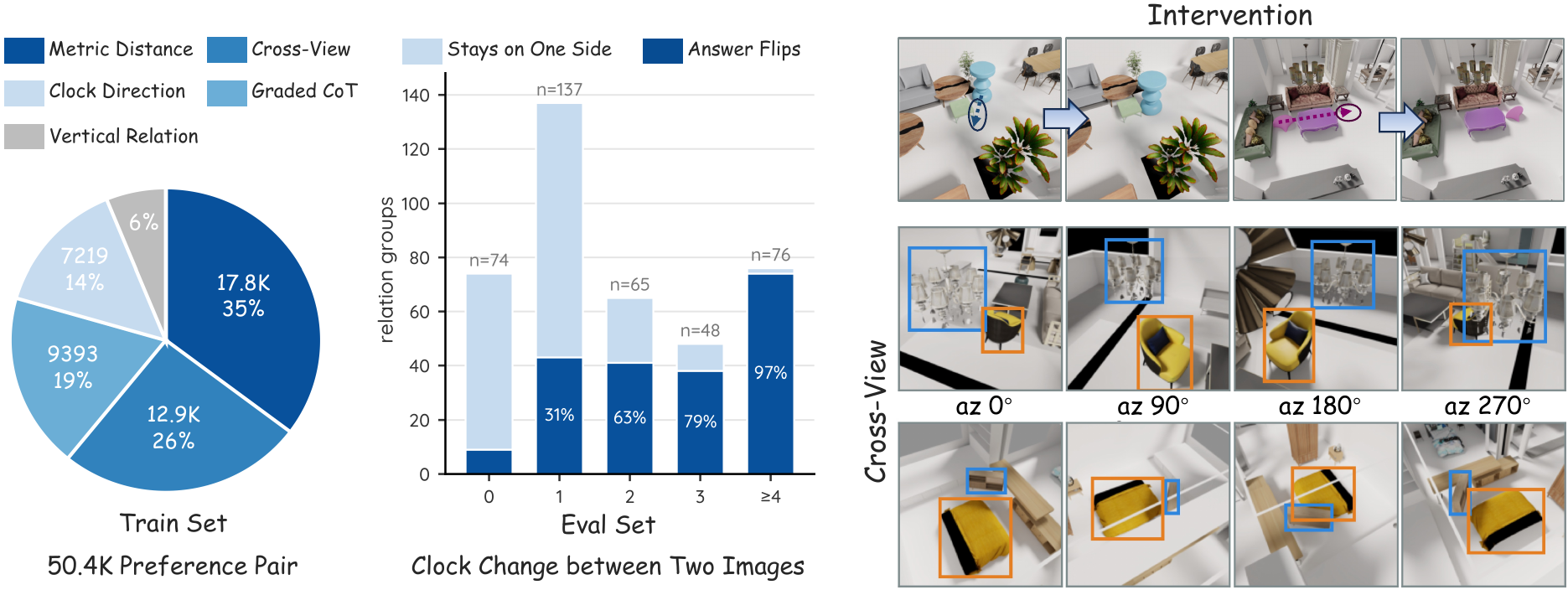}
    \caption{\textbf{One construction across training and evaluation.}
    Left, \data{} by pair family, with geometric displacement annotations for measured pairs.
    Middle, \bench{} groups held-out scenes by the clock change between two images and reports how often the answer flips.
    Right, an intervention moves one object and re-renders the scene from up to four views.}
    \label{fig:bench}
\end{figure*}

\paragraph{Magnitude-aware preferences.}
We initialize $\pi_\theta$ from SFT and freeze a copy as $\pi_{\rm ref}$.
Records $z=((x_w,y_w),(x_l,y_l))$, with image--question inputs $x$,
split into measured $\mathcal D_{\rm geom}$ and ordering-only
$\mathcal D_{\rm bin}$. Measured pairs must parse correctly;
using each branch's own serialized truth, we require
\begin{equation}
m_{J_z}\!\left(\Pi_{J_z}(y_w),r_{w,J_z}^\star\right)=0,
\qquad
m_z:=m_{J_z}\!\left(\Pi_{J_z}(y_l),r_{l,J_z}^\star\right)>0,
\label{eq:measured-pair}
\end{equation}
Ordering-only pairs, including vertical comparisons, require a valid
preference, not an exact preferred answer. Both use
$h_\theta(z)\!
=\log\!\bigl(\pi_\theta(y_w\!\mid\!x_w)/\pi_{\rm ref}(y_w\!\mid\! x_w)\bigr)\!
-\log\bigl(\pi_\theta(y_l\!\mid\! x_l)/\pi_{\rm ref}(y_l\!\mid\! x_l)\bigr)$. This is a within-input response contrast or an mDPO-style image
contrast~\citep{wangMDPOConditionalPreference2024}; the latter does not
identify a cross-input reward difference. All sequence scores sum
answer tokens through termination, excluding prompt and padding.
We set
\begin{equation}
\gamma_z=\beta m_z\ \ (z\in\mathcal D_{\rm geom}),\quad
{\rm or}
\quad
\gamma_z=0\ \ (z\in\mathcal D_{\rm bin}),
\qquad \beta>0,
\label{eq:mixed-offset}
\end{equation}
Zero offset denotes ordering-only supervision in the pair branch, not zero
geometric supervision in the full objective and not zero error.
Nonempty strata $k$ index task, pair side and supervision type; fixed weights
$w_k>0$, $\sum_k w_k=1$, define
$\mathcal D_{\rm mix}=\sum_k w_k\operatorname{Unif}(\mathcal D_k)$ and
\begin{equation}
\mathcal L_{\rm pair}
=-\mathbb E_{z\sim\mathcal D_{\rm mix}}
\log\sigma\!\left(\beta h_\theta(z)-\gamma_z\right),
\label{eq:gauge}
\end{equation}
At fixed $\beta$, larger measured errors demand larger gaps at a given
positive loss level. Zero offsets recover the DPO form in the pair branch; the full objective
retains direct and intervention terms.

\paragraph{Canonical prediction.}
For (R2), fix a finite grid $\mathcal D_Q\subset\mathbb R_{>0}$ before
training independently of test truths, with
$\mathcal C=\mathcal C_{12}\times\mathcal D_Q\subseteq\mathcal R_Q$
and group truths $r_G^\star=Q(r_{{\rm phys},G}^\star)\in\mathcal C$.
A unique terminated relation sequence $\psi(r)$ satisfies $\Pi(\psi(r))=r$.
Under a fixed prompt $q_{\rm rel}$ and $x_v^{\rm rel}=(I_v,q_{\rm rel})$,
inference returns $\psi(\hat r_v)$ under a fixed tie rule, without
test truth or GT rationales, where
\begin{equation}
s_\theta(x_v^{\rm rel},r)=\log\pi_\theta(\psi(r)\mid x_v^{\rm rel}),
\qquad
\hat r_v\in\arg\max_{r\in\mathcal C}s_\theta(x_v^{\rm rel},r),
\label{eq:canonical-predict}
\end{equation}

\paragraph{Direct supervision of the worst view.}
Set $\Delta(r,\!r^\star)\!=\!\delta_0\,\mathbb I[r\!\ne \!r^\star]\!+\!\xi\,m(r,\!r^\star)$,
where $\delta_0\!>\!0$ is a classification margin, $\xi\!>\!0$ scales geometry.
The most violated view--candidate comparison defines:
\begin{align}
\ell_{\rm dir}(G)
&=\left[
\max_{\substack{1\le v\le V_G,\text{ } r\in\mathcal C\setminus\{r_G^\star\}}}
\left\{s_\theta(x_v^{\rm rel},r)-s_\theta(x_v^{\rm rel},r_G^\star)
+\Delta(r,r_G^\star)\right\}\right]_+,
\label{eq:direct-group}\\
\mathcal L_{\rm dir}&=\mathbb E_{G\sim\mathcal D_{\rm grp}}\ell_{\rm dir}(G).
\label{eq:direct-loss}
\end{align}
\begin{proposition}[Direct control of canonical group predictions]
\label{prop:direct-group}
For a nonempty finite group with finite scores,
$r_G^\star\in\mathcal C$, and $|\mathcal C|\ge2$,
Eq.~\ref{eq:canonical-predict} satisfies
\begin{equation}
\delta_0\,\mathbb I[\exists v:\hat r_v\ne r_G^\star]
+\xi\max_v m(\hat r_v,r_G^\star)\le\ell_{\rm dir}(G),
\label{eq:direct-bound}
\end{equation}
Thus $\ell_{\rm dir}(G)<\delta_0$ guarantees correctness, with a unique
truth maximizer at zero loss; correct predictions may still incur loss.
Also $\frac{\xi}{2}\max_{u,v}m(\hat r_u,\hat r_v)\le\ell_{\rm dir}(G)$.
\end{proposition}

\subsection{Linking Intervention Responses to Geometry}
\label{sec:intervention}

\providecommand{\gain}[1]{\textcolor{teal!70!black}{\scriptsize$+#1$}}
\providecommand{\drop}[1]{\textcolor{red!75!black}{\scriptsize$-#1$}}

\begin{table}[t]
\vspace{-12pt}
  \centering\scriptsize
  \setlength{\tabcolsep}{1.2pt}\renewcommand{\arraystretch}{.99}
    \caption{\textbf{Spatial understanding results.}
    $^{\S}$ marks models evaluated without their native additional visual cues for fair comparison. $^{\dag}$ marks constant outputs.}
  \label{tab:main}
  \resizebox{\linewidth}{!}{%
  \begin{tabular}{l c *{3}{c} c *{2}{c} *{2}{c} c c *{3}{c}}
  \toprule
  & & \multicolumn{3}{c}{\textbf{MSMU}}
    & \textbf{QSpat}
    & \multicolumn{2}{c}{\textbf{SURDS}}
    & \multicolumn{2}{c}{\textbf{SpatialRGPT}}
    & \textbf{3DSR}
    & \textbf{BLINK}
    & \multicolumn{3}{c}{\textbf{Constancy-Bench}} \\
  \cmidrule(lr){3-5}
  \cmidrule(lr){6-6}
  \cmidrule(lr){7-8}
  \cmidrule(lr){9-10}
  \cmidrule(lr){11-11}
  \cmidrule(lr){12-12}
  \cmidrule(lr){13-15}

  Model & \#P
  & \shortstack{dist.}
  & \shortstack{width}
  & \shortstack{height}
  & \shortstack{$\delta_2$}
  & \shortstack{dist.}
  & \shortstack{depth}
  & \shortstack{Quan.}
  & \shortstack{Qual.}
  & \shortstack{acc.}
  & \shortstack{acc.}
  & \shortstack{Rank\\ $\rho\!\uparrow$}
  & \shortstack{VisC.\\ \%\,$\uparrow$}
  & \shortstack{PairAcc.\\ \%\,$\uparrow$} \\

  \midrule
  \rowcolor{gray!18}
  \multicolumn{15}{l}{\emph{Frontier models}}\\

Claude Opus\,4.8 & --
& 2.5 & 1.1 & 2.2
& 76.2
& 11.7 & 14.7
& 41.5 & 55.8
& 66.4
& 70.1
& .127 & 52.7 & 62.0 \\

GPT-4o & --
& 2.5 & 0.0 & 5.5
& 47.5
& 2.7 & 0.0
& 41.7 & 74.3
& 55.3
& 56.2
& .136 & 32.4 & 45.4 \\

Gemini-2.5-Pro & --
& 7.5 & 6.7 & 18.7
& 45.5
& 29.5 & 34.3
& 38.1 & 80.9
& 54.2
& 54.1
& .227 & 37.8 & 64.4 \\

Kimi-K2.6 & 1T
& 15.0 & 9.0 & 18.7
& 64.4
& 48.5 & 53.1
& 42.0 & 89.3
& 62.8
& 72.8
& .200 & 46.0 & 68.8 \\

\midrule
\rowcolor{red!8}
\multicolumn{15}{l}{\emph{Prior open spatial VLMs}}\\

SpaceOm & 3B
& 15.0 & 5.6 & 3.3
& 50.5
& 3.0 & 9.5
& 24.0 & 47.6
& 49.9
& 47.1
& $-.094$ & 84.6 & 9.3 \\

SpatialLadder & 3B
& 17.5 & 1.1 & 6.6
& 50.5
& 0.0 & 0.0
& 17.5 & 43.2
& 48.2
& 44.9
& $-.029$ & 87.2 & 2.9 \\

VST-RL & 7B
& 30.0 & 30.3 & 50.5
& 51.5
& 7.3 & 4.6
& 39.1 & 61.2
& 57.7
& 50.8
& .264 & 80.4 & 56.1 \\

SpatialReasoner\rlap{$^{\S}$} & 8B
& 25.0 & 4.5 & 9.9
& 49.5
& 19.9 & 23.8
& 25.3 & 69.7
& 56.8
& 42.8
& .277 & 68.9 & 59.0 \\

SpatialRGPT\rlap{$^{\S}$} & 8B
& 25.0 & 2.2 & 4.4
& 55.4
& 2.3 & 1.5
& 25.7 & 60.1
& 53.4
& 45.9
& .000 & 85.8 & 45.9 \\

SD-VLM & 7B
& 47.5 & 47.2 & 68.1
& 27.7
& 0.01 & 0.0
& 22.5 & 49.2
& 26.6
& 16.5
& .000\rlap{$^{\dag}$} & 100.0\rlap{$^{\dag}$} & 0.0 \\

\midrule
\rowcolor{myblue!20}
\multicolumn{15}{l}{\emph{Ours --- Qwen family}}\\

\quad Qwen2.5-VL & 7B
& 17.5 & 3.4 & 12.1
& 42.6
& 20.0 & 22.3
& 25.0 & 70.0
& 44.9
& 49.5
& .115 & 84.5 & 36.6 \\

\quad {}\,GaugeSFT & 7B
& 47.5 & 48.3 & 67.0
& 45.5
& 22.4 & 24.3
& 33.5 & 74.7
& 48.2
& 50.5
& .109 & 81.1 & 54.6 \\

\rowcolor{blue!8}
\quad {}\,\method{} & 7B
& \textbf{62.5} & \textbf{55.1} & \textbf{71.4}
& \textbf{64.4}
& \textbf{30.4} & \textbf{35.7}
& \textbf{43.1} & \textbf{81.6}
& \textbf{56.9}
& \textbf{56.7}
& \textbf{.279} & \textbf{91.2} & \textbf{66.3} \\

\quad $\Delta_{\text{vs SFT}}$ & &
\gain{15.0} & \gain{6.7} & \gain{4.4}
& \gain{18.9}
& \gain{8.0} & \gain{11.4}
& \gain{9.6} & \gain{6.9}
& \gain{8.7}
& \gain{6.2}
& \gain{.170} & \gain{10.1} & \gain{11.7} \\

\midrule
\rowcolor{myblue!20}
\multicolumn{15}{l}{\emph{Ours --- GLM family}}\\

\quad GLM-4.1V & 9B
& 30.0 & 1.1 & 9.9
& 43.6
& 29.5 & 26.4
& 31.9 & 71.2
& 48.2
& 51.0
& .290 & 62.8 & 65.4 \\

\quad {}\,GaugeSFT & 9B
& 67.5 & 57.3 & 67.0
& 49.5
& 33.9 & 29.1
& 38.4 & 68.8
& 46.7
& 56.2
& .235 & 81.8 & 67.3 \\

\rowcolor{blue!8}
\quad {}\,\method{} & 9B
& \textbf{72.5} & \textbf{65.2} & \textbf{75.8}
& \textbf{58.4}
& \textbf{35.2} & \textbf{31.3}
& \textbf{43.5} & \textbf{78.7}
& \textbf{50.8}
& \textbf{58.2}
& \textbf{.311} & \textbf{90.5} & \textbf{69.8} \\

\quad $\Delta_{\text{vs SFT}}$ & &
\gain{5.0} & \gain{7.9} & \gain{8.8}
& \gain{8.9}
& \gain{1.3} & \gain{2.2}
& \gain{5.1} & \gain{9.9}
& \gain{4.1}
& \gain{2.0}
& \gain{.076} & \gain{8.7} & \gain{2.5} \\

\midrule

\rowcolor{myblue!20}
\multicolumn{15}{l}{\emph{Ours --- Mistral family}}\\

\quad Pixtral  & 12B
& 10.0 & 4.5 & 4.4
& 8.9
& 21.3 & 15.8
& 27.6 & 64.8
& 44.0
& 37.2
& .096 & 85.1 & 34.6 \\

\quad {}\,GaugeSFT & 12B
& 50.0 & 53.9 & 69.2
& 7.9
& 22.0 & 18.8
& 18.0 & 60.0
& 43.5
& 37.0
& .000\rlap{$^{\dag}$} & 100.0\rlap{$^{\dag}$} & 32.2 \\

\rowcolor{blue!8}
\quad {}\,\method{} & 12B
& \textbf{52.5} & \textbf{56.2} & \textbf{75.8}
& \textbf{11.9}
& \textbf{28.5} & \textbf{23.7}
& \textbf{28.0} & \textbf{67.0}
& \textbf{48.0}
& \textbf{42.6}
& \textbf{.138} & 99.3 & \textbf{41.0} \\

\quad $\Delta_{\text{vs SFT}}$ & &
\gain{2.5} & \gain{2.2} & \gain{6.6}
& \gain{4.0}
& \gain{6.5} & \gain{4.9}
& \gain{10.0} & \gain{7.0}
& \gain{4.5}
& \gain{5.6}
& \gain{.138} & \drop{0.7} & \gain{8.8} \\

\bottomrule
\end{tabular}}
\vspace{-3pt}
\end{table}

\begin{figure*}[t]
    \vspace{-1pt}
    \centering
    \includegraphics[width=\linewidth]{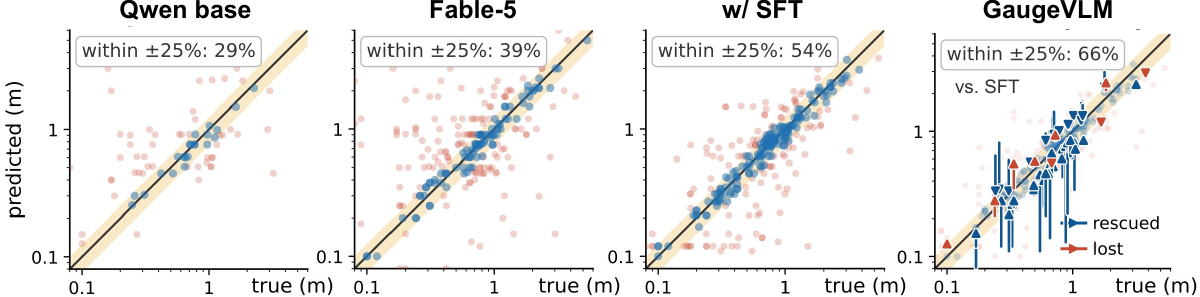}
    \caption{\textbf{Predicted versus true values across all MSMU metric tasks.} GaugeVLM achieves the best calibration even outperforming the Claude-Fable-5.
    Missing or unparseable numerical outputs are not plotted, so point counts vary across models.
    ``Within $\pm25\%$'' denotes $|\hat d-d^\star|/d^\star\le0.25$.}
    \label{fig:compare_fable5}
\end{figure*}

Correct rankings do not determine relative gaps across scene changes.
A block $B=\{(x_{i,v}^{\rm rel},r_i):i=0,\ldots,K;\ v=1,\ldots,V_B\}$
renders an original state and $K\ge2$ interventions from $V_B\ge2$
common cameras. Each state forms a direct group $G_i$, with
$r_i\in\mathcal C$ and $m_k=m(r_0,r_k)>0$ taking at least two distinct values.
Suppressing the canonical-input superscript, define
\begin{equation}
H_{v,k}=\underbrace{s_\theta(x_{0,v},r_0)-s_\theta(x_{0,v},r_k)}_{\text{original-state answer log-odds}}
+\underbrace{s_\theta(x_{k,v},r_k)-s_\theta(x_{k,v},r_0)}_{\text{intervened-state answer log-odds}},
\label{eq:interaction}
\end{equation}
This intervention contrast cancels separable score components $a(x)+b(r)$,
including input normalizers. Each direct gap includes classification margin $\delta_0$;
we fit the excess response:
\begin{equation}
\ell_{\rm int}(B)=\frac1{V_B}\sum_v\min_{\eta\ge0}
\frac1K\sum_{k=1}^K\big(H_{v,k}-2\delta_0-\eta m_k\big)^2,
\qquad \mathcal L_{\rm int}=\mathbb E_{B\sim\mathcal D_{\rm int}}\ell_{\rm int}(B),
\label{eq:intervention}
\end{equation}
Each camera fits one scale across a block's interventions, allowing
view-dependent confidence. This prescribed proportionality concerns
the sum of two gaps without requiring equality.
If all $\ell_{\rm dir}(G_i)$ and $\ell_{\rm int}(B)$ vanish, every canonical
prediction is correct and $H_{v,k}\!-\!2\delta_0\!=\!\eta_v m_k$ with $\eta_v\!\ge\!2\xi$.

\paragraph{Joint objective.}
The preference stage combines three separately averaged losses:
\begin{equation}
\mathcal L_{\rm \obj}=\mathcal L_{\rm pair}
+\lambda_{\rm dir}\mathcal L_{\rm dir}
+\lambda_{\rm int}\mathcal L_{\rm int},
\qquad\lambda_{\rm dir},\lambda_{\rm int}>0,
\label{eq:full}
\end{equation}
Only the pair term uses reference-relative scores. Fixed manifests define
group and block distributions, with every state directly supervised.
Guarantees are pointwise on evaluated groups and blocks for the specified
predictor; matched experiments
assess optimization and generalization.

\subsection{The Gauge Engine}
\label{sec:data}
\paragraph{From scene geometry to supervision.}
The rendering pipeline places Objaverse-LVIS or 3D-FUTURE objects in 3D-FRONT scenes using BlenderProc~\citep{deitkeObjaverseUniverseAnnotated2022a,fu3DFUTURE3DFurniture2020,fu3DFRONT3DFurnished2021,denningerBlenderProc2019}. Object poses determine the anchor-relative clock and centroid distance. Response-side rejections change the serialized relation; image-side pairs move the target and retain the original response. Each measured error is evaluated against its branch's truth. Vertical and reasoning-quality comparisons without a complete geometric target use the ordering-only branch. The supplied \data{} pool reports $50.4$K pairs across five task families and $1{,}174$ image-side scene pairs. These are pair-pool counts, not counts of the new intervention blocks.

\renewcommand{\gain}[1]{\textcolor{teal!70!black}{+#1}}
\renewcommand{\drop}[1]{\textcolor{red!75!black}{-#1}}

\begin{table}[t]\vspace{-10pt}
    \centering
    \scriptsize
    \setlength{\tabcolsep}{4.pt}
    \renewcommand{\arraystretch}{.96}

    \caption{\textbf{General Vision-Language Understanding Ability Results.} Average{$^{\ddag}$} denotes the mean of the seven percentage metrics, excluding MME.}
    \label{tab:general}

    \resizebox{.98\linewidth}{!}{%
        \begin{tabular}{l c c c c c c c c c c c}
            \toprule
            Model
            & \#P
            & MME$^{\mathrm{P}}$
            & MME$^{\mathrm{R}}$
            & POPE
            & MMStar
            & AI2D
            & SEED
            & SQA
            & HallB
            & MMMU
            & Average\rlap{$^{\ddag}$} \\
            \midrule

            Qwen2.5-VL
            & 7B
            & 1606
            & 622
            & 83.7
            & 58.1
            & 78.8
            & 73.0
            & 72.7
            & 63.3
            & 42.8
            & 67.5 \\

            \rowcolor{blue!8}
            GaugeVLM
            & 7B
            & \textbf{1640}
            & \textbf{629}
            & \textbf{85.9}
            & \textbf{60.3}
            & \textbf{80.1}
            & \textbf{75.0}
            & \textbf{78.9}
            & \textbf{64.5}
            & \textbf{44.4}
            & \textbf{69.9} \\

            \quad $\Delta_{\text{vs base}}$
            &
            & \gain{34}
            & \gain{7}
            & \gain{2.2}
            & \gain{2.2}
            & \gain{1.3}
            & \gain{2.0}
            & \gain{6.2}
            & \gain{1.2}
            & \gain{1.6}
            & \gain{2.4} \\

            \midrule
            GLM-4.1V
            & 9B
            & 1593
            & 555
            & 84.6
            & 28.4
            & 57.1
            & 73.1
            & 28.7
            & 49.8
            & 36.6
            & 51.2 \\

            \rowcolor{blue!8}
            GaugeVLM
            & 9B
            & \textbf{1595}
            & \textbf{560}
            & \textbf{88.9}
            & \textbf{32.1}
            & \textbf{64.3}
            & \textbf{74.0}
            & \textbf{42.7}
            & \textbf{52.2}
            & \textbf{40.7}
            & \textbf{56.4} \\

            \quad $\Delta_{\text{vs base}}$
            &
            & \gain{2}
            & \gain{5}
            & \gain{4.3}
            & \gain{3.7}
            & \gain{7.2}
            & \gain{0.9}
            & \gain{14.0}
            & \gain{2.4}
            & \gain{4.1}
            & \gain{5.2} \\

            \midrule
 
            Pixtral
            & 12B
            & 1598
            & 411
            & 85.4
            & 47.9
            & 78.1
            & 74.5
            & 81.9
            & 60.0
            & 46.0
            & 67.7 \\

            \rowcolor{blue!8}
            GaugeVLM
            & 12B
            & \textbf{1610}
            & \textbf{411}
            & \textbf{85.9}
            & \textbf{50.0}
            & \textbf{79.9}
            & \textbf{75.1}
            & \textbf{84.2}
            & \textbf{60.5}
            & \textbf{49.4}
            & \textbf{69.3} \\

            \quad $\Delta_{\text{vs base}}$
            &
            & \gain{12}
            & \gain{0.0}
            & \gain{0.5}
            & \gain{2.1}
            & \gain{1.8}
            & \gain{0.6}
            & \gain{2.3}
            & \gain{0.5}
            & \gain{3.4}
            & \gain{1.6} \\

            \bottomrule
        \end{tabular}%
    }
\end{table}
\begin{figure}[t]
    \centering
    \includegraphics[width=0.99\linewidth]{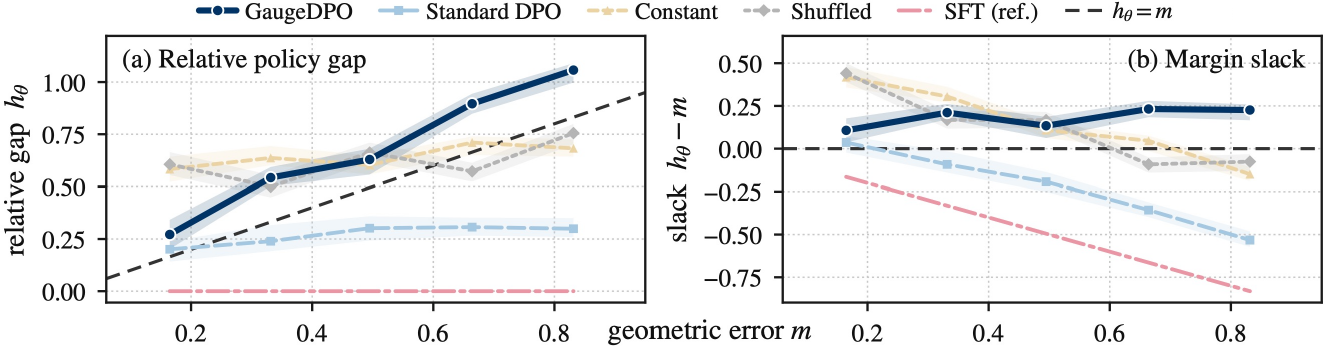}
    \caption{\textbf{Measured margins induce graded policy separation.} On held-out scenes, \model{} scales its reference-relative gap with geometric error (a) and maintains positive slack (b). SFT is the reference ($h_{\mathrm{SFT}}=0$); bands show 95\% paired scene-bootstrap CIs.}
    \label{fig:margin_h}
\end{figure}

\paragraph{Matched states and cameras.}
For the revised objective, each block holds the anchor, object identities, scene context and camera set fixed while changing the target relation at several measured magnitudes. A fixed state supplies a camera group; all states together supply the interaction profile. Blocks require co-visibility, an identifiable anchor frame and sufficient metric-scale evidence. Base scenes and related object assets are separated before generating split-specific states and views. The block manifest records physical and serialized truths, so quantization collisions and out-of-domain cases are explicit.
Full derivations are provided in App.~\ref{app:derivations}, with data construction details in App.~\ref{app:data}.

\section{Experiments}
\label{sec:experiments}

\subsection{Experimental Setup}
\label{sec:exp-setup}

\begin{table}[t]
\vspace{-10pt}
    \centering
    \scriptsize
    \setlength{\tabcolsep}{3.6pt}
    \renewcommand{\arraystretch}{.96}

    \caption{\textbf{Ablation of preference margin design and SFT initialization.} Zero-, constant-, shuffled- and measured-margin rows retain the same direct/profile supervision; only the pair offset changes.}
    \label{tab:abla}

    \resizebox{\linewidth}{!}{%
        \begin{tabular}{c l l c c c c c c c}
            \toprule
            &
            Settings &
            Margin $m$ &
            \shortstack{\textbf{MSMU}\\[-.3pt]dist.} &
            \shortstack{\textbf{QSpat}\\[-1.3pt]$\delta_2$} &
            \shortstack{\textbf{SRGPT}\\[-.3pt]Quan} &
            \shortstack{\textbf{3DSR}\\[1.3pt]acc.} &
            \shortstack{\textbf{BLINK}\\[1.3pt]acc.} &
            \shortstack{\textbf{Rank}\\[-.3pt]$\rho\uparrow$} &
            \shortstack{\textbf{PairAcc.}\\[-.3pt]\% $\uparrow$}
            \\
            \midrule

            (a)\quad &
            SFT Initialization &
            none &
            47.5 &
            45.5 &
            33.5 &
            48.2 &
            50.5 &
            .109 &
            54.6
            \\

            \midrule

            (b)\quad &
            {}+\,Zero-offset pair loss &
            $m \equiv 0$ &
            50.0 &
            49.5 &
            35.4 &
            50.3 &
            51.9 &
            .119 &
            56.6
            \\

            (c)\quad &
            {}+\,Constant margin &
            $m \equiv \bar{m}$ &
            52.5 & 
            52.5 & 
            37.2 & 
            51.6 & 
            53.1 & 
            .146 & 
            58.0 \\
            
            (d)\quad &
            {}+\,Shuffled margin &
            $m_{\sigma(i)}$ &
            52.5 & 
            51.5 & 
            36.6 & 
            51.8 & 
            52.6 & 
            .132 & 
            57.6 \\
            
            (e)\quad &
            {}+\,Reward margin &
            $\hat{m}_{\phi}$ (estimated) &
            57.5 & 
            56.4 & 
            39.7 & 
            53.7 & 
            52.9 & 
            .142 & 
            61.0 \\

            \rowcolor{blue!8}
            (f)\quad &
            {}+\,\obj{} &
            measured (Eq.~\ref{eq:margin}) &
            \textbf{62.5} &
            \textbf{64.4} &
            \textbf{43.1} &
            \textbf{56.9} &
            \textbf{56.7} &
            \textbf{.279} &
            \textbf{66.3}
            \\

            \midrule

            (g)\quad &
            \obj{} w/o SFT &
            measured (Eq.~\ref{eq:margin})&
            40.0 &
            54.5 &
            35.8 &
            49.4 &
            51.4 &
            .204 &
            52.2
            \\

            \bottomrule
        \end{tabular}%
        }
\vspace{-4pt}
\end{table}

\newcommand{\bc}[1]{\cellcolor{blue!8}#1}
\newcommand{\hc}[1]{\cellcolor{myblue!20}#1}
\begin{table}[t]
    \vspace{-10pt}
    \centering

    \begin{minipage}[t]{0.47\linewidth}
        \centering
        \scriptsize
        \setlength{\tabcolsep}{1.8pt}
        \renewcommand{\arraystretch}{1.15}

        \caption{\textbf{Data efficiency of geometric magnitude supervision.} With only 25\% preference data, GaugeDPO surpasses full-data DPO.}
        \label{tab:data_efficiency}

        \resizebox{\linewidth}{!}{%
        \begin{tabular}{c l c c c c c}
            \toprule
            \textbf{Data} &
            \textbf{Method} &
            \shortstack{\textbf{MSMU}\\[1pt]dist.} &
            \shortstack{\textbf{QSpat}\\[-.2pt]$\delta_2$} &
            \shortstack{\textbf{SRGPT}\\[1pt]Quan.} &
            \shortstack{\textbf{3DSR}\\[2.5pt]acc.} &
            \shortstack{\textbf{BLINK}\\[2.5pt]acc.} \\
            \midrule

            \multirow{2}{*}{10\%}
            & DPO
            & 47.5 & 46.5 & 33.7 & 48.8 & 50.8 \\
            & \bc{\obj{}}
            & \bc{\textbf{50.0}}
            & \bc{\textbf{48.5}}
            & \bc{\textbf{34.5}}
            & \bc{\textbf{49.6}}
            & \bc{\textbf{51.5}} \\

            \midrule

            \multirow{2}{*}{25\%}
            & DPO
            & 50.0 & 47.5 & 34.2 & 49.3 & 51.2 \\
            & \hc{\obj{}\rlap{$^{\star}$}}
            & \hc{\textbf{52.5}}
            & \hc{\textbf{51.5}}
            & \hc{\textbf{36.1}}
            & \hc{\textbf{51.0}}
            & \hc{\textbf{52.6}} \\

            \midrule

            \multirow{2}{*}{50\%}
            & DPO
            & 50.0 & 48.5 & 34.9 & 49.8 & 51.6 \\
            & \bc{\obj{}}
            & \bc{\textbf{57.5}}
            & \bc{\textbf{57.4}}
            & \bc{\textbf{39.5}}
            & \bc{\textbf{53.4}}
            & \bc{\textbf{54.7}} \\

            \midrule

            \multirow{2}{*}{100\%}
            & DPO
            & 50.0 & 49.5 & 35.4 & 50.3 & 51.9 \\
            & \bc{\obj{}}
            & \bc{\textbf{62.5}}
            & \bc{\textbf{64.4}}
            & \bc{\textbf{43.1}}
            & \bc{\textbf{56.9}}
            & \bc{\textbf{56.7}} \\

            \bottomrule
        \end{tabular}%
        }
    \end{minipage}
    \hfill
    \begin{minipage}[t]{0.51\linewidth}
        \centering
        \scriptsize
        \setlength{\tabcolsep}{1.0pt}
        \renewcommand{\arraystretch}{0.95}

        \caption{\textbf{Sensitivity analysis.}
        $\lambda_{\rm dir}$ weights direct supervision;
        $\lambda_{\rm int}$ weights intervention fitting.
        The other weight stays at its shaded default.}
        \label{tab:hparams}

        \resizebox{\linewidth}{!}{%
        \begin{tabular}{c c c c c c c c}
            \toprule
            \textbf{Para.} &
            \textbf{Value} &
            \shortstack{\textbf{MSMU}\\[1pt]dist.} &
            \shortstack{\textbf{QSpat}\\[-.2pt]$\delta_2$} &
            \shortstack{\textbf{SRGPT}\\[1pt]Quan.} &
            \shortstack{\textbf{3DSR}\\[2.5pt]acc.} &
            \shortstack{\textbf{BLINK}\\[2.5pt]acc.} &
            \shortstack{\textbf{PairAcc.}\\[1pt]\%\,$\uparrow$}
            \\
            \midrule

            & $0$
            & 57.5 & 60.4 & 40.6 & 54.8 & 55.3 & 64.4 \\

            & $0.03$
            & 60.0 & 63.4 & 42.4 & 56.1 & 56.3 & 65.9 \\

            $\lambda_{\rm dir}$
            & \bc{$0.1$}
            & \bc{\textbf{62.5}}
            & \bc{\textbf{64.4}}
            & \bc{43.1}
            & \bc{\textbf{56.9}}
            & \bc{56.7}
            & \bc{\textbf{66.3}} \\

            & $0.3$
            & \textbf{62.5} & \textbf{64.4} & \textbf{43.3} & 56.7 & \textbf{56.8} & 65.9 \\

            & $1.0$
            & 60.0 & 62.4 & 42.1 & 55.8 & 56.0 & 65.4 \\

            \midrule

            & $0$
            & 60.0 & 63.4 & 42.5 & 56.1 & 56.2 & 60.0 \\

            & $0.25$
            & \textbf{62.5} & \textbf{64.4} & 42.9 & 56.7 & 56.5 & 64.9 \\

            $\lambda_{\rm int}$
            & \bc{$0.5$}
            & \bc{\textbf{62.5}}
            & \bc{\textbf{64.4}}
            & \bc{\textbf{43.1}}
            & \bc{56.9}
            & \bc{\textbf{56.7}}
            & \bc{\textbf{66.3}} \\

            & $0.75$
            & \textbf{62.5} & \textbf{64.4} & 43.0 & \textbf{57.0} & \textbf{56.7} & \textbf{66.3} \\

            & $1$
            & 60.0 & 63.4 & 42.6 & 56.6 & 56.3 & 65.4 \\

            \bottomrule
        \end{tabular}%
        }
    \end{minipage}
\end{table}

\paragraph{Models and benchmarks.}
We instantiate \model{} on Qwen2.5-VL-7B-Instruct, GLM-4.1V-9B, and Pixtral-12B~\citep{baiQwen25VLTechnicalReport2025,teamGLM45VGLM41VThinkingVersatile2026,agrawal2024pixtral12b}, using Qwen as the main backbone. We compare with SpaceOm, SpatialLadder, VST-RL, SpatialReasoner, SpatialRGPT, SD-VLM~\citep{spaceom, liSpatialLadderProgressiveTraining2025, yangVisualSpatialTuning2025, shenFineGrainedPreferenceOptimization2026, chengSpatialRGPTGroundedSpatial, chenSDVLMSpatialMeasuring2025}, and frontier models Kimi-K2.6, Claude-Opus-4.8, GPT-4o, Gemini-2.5-Pro~\citep{moonshotKimiK262026, anthropicClaudeOpus482026, openaiGPT4oSystemCard2024, googleGemini25ProModelCard2025}.
Evaluation spans benchmarks from diverse domains to assess both in-domain reasoning and cross-domain generalization, including general spatial understanding on MSMU, QSpatial$^+$ (QSpat.), SpatialRGPT-Bench, 3DSRBench, and BLINK~\citep{chenSDVLMSpatialMeasuring2025, liaoReasoningPathsReference2024, chengSpatialRGPTGroundedSpatial, ma3DSRBenchComprehensive3D2024, fuBLINKMultimodalLarge2024}, with driving-oriented reasoning on SURDS~\citep{guoSURDSBenchmarkingSpatial2025}.

\paragraph{Evaluation metrics.}
\label{sec:eval-metrics}
In Tab.~\ref{tab:main}, \textit{Rank} is Spearman correlation between predicted and true distances; \textit{VisC}. is prediction agreement (\%) under photometric changes at fixed geometry and camera. Let $\mathcal P,\mathcal G,\mathcal B$ contain relocation pairs, fixed-truth camera groups and canonical multi-state blocks. Set $C_S=\prod_{j\in S}\mathbb I[\hat a_j=a_j^\star]$ for $S\in\mathcal P\cup\mathcal B$. For the tolerance-based view audit, $C_G^{\rm tol}$ and $A_G^{\rm tol}$ indicate valid groups passing correctness and agreement thresholds, respectively; overbars denote manifest means. \textit{ViewAgree}$_{\rm tol}$ is $100\overline{A^{\rm tol}}_{\mathcal G}$, \textit{WrongAgree}$_{\rm tol}$ is
$100\overline{A^{\rm tol}(1-C^{\rm tol})}_{\mathcal G}$, and \textit{Mean worst} $m$ is
$\overline{\max_{i,v}m(\hat r_{i,v},r_i)}_{\mathcal B}$.
We average $E_{\rm shape}$ over eligible block--camera profiles, and main 7B PairAcc uses $205$ pairs. With $U_v=H_v-2\delta_0\mathbf1$, $m_B=(m_1,\ldots,m_K)$ and the fitted
$\eta_v^*$ from Eq.~\ref{eq:intervention},
\begin{align}
(\mathrm{PairAcc},\mathrm{GroupAcc}_{\rm tol},\mathrm{BlockAcc})
&=100\big(\overline C_{\mathcal P},\overline{C^{\rm tol}}_{\mathcal G},\overline C_{\mathcal B}\big),
\label{eq:pairacc}\\
E_{\rm shape}(B,v)
&=\frac{\|U_v-\eta_v^*m_B\|_2^2}{\|U_v\|_2^2},
\qquad \|U_v\|_2>\tau_U>0,
\label{eq:profile-evaluation}
\end{align}

\paragraph{Training and scoring units.}
All backbones use spatial SFT on $25{,}378$ QA turns followed by preference optimization, with a frozen SFT checkpoint as reference. In the SFT stage, we use the OSD dataset following SpatialRGPT together with 10K samples drawn from the MSMU training set, and fine-tune the language-model parameters with LoRA of rank $128$ and $\alpha_{\rm LoRA}=128$, while the vision encoder is frozen. The SFT adapter is merged before the subsequent \obj{} stage, which uses our constructed \date{} data and performs full-parameter fine-tuning of the model and vision encoder. The recorded defaults are $\lambda_{\rm dir}=.1$, $\lambda_{\rm int}=.5$, $\beta=.3$, geometric $\alpha=.5$, and $\kappa=\log5$, with $\tau_c = 1/6$ and $\tau_d = 0.25$ in normalized ring/log-distance units. Chunking reduces peak activation memory, not
candidate-scoring cost. All training is conducted on four NVIDIA A100 80GB GPUs. Full training details and settings are provided in App.~\ref{app:training}.

\subsection{Main Results}
\label{sec:exp-results}

\textbf{Broad Gains and Competitive Performance.}
Across backbones, we improve all $30/30$ established spatial-task metrics over both base and SFT checkpoints (Tab.~\ref{tab:main}). The VisC exception ($-0.7$) accompanies recovery from SFT output collapse. QSpatial gains reach $18.9$ points from preference learning versus $3.0$ from imitation. Gains also extend to driving and embodied 3D benchmarks. At 7B, \model{} leads prior open spatial VLMs on $9/10$ established metrics under the evaluated input protocol, trailing VST-RL by $<1$ point on 3DSRBench. It leads the listed frontier models on all three MSMU tasks and exceeds GPT-4o and Gemini-2.5-Pro on all reported metrics. Fig.~\ref{fig:compare_fable5} additionally compares distance predictions with Claude-Fable-5~\citep{anthropicClaudeFable52026}. Additionally, across three families in Tab.~\ref{tab:general}, all metrics from eight general benchmarks~\citep{fuMMEComprehensiveEvaluation2025,liEvaluatingObjectHallucination2023,chenAreWeRight,kembhavi2016ai2d,liSEEDBench2BenchmarkingMultimodal2023,luLearnExplainMultimodal2022,guanHallusionBenchAdvancedDiagnostic2024,yueMMMUMassiveMultidiscipline2024} match or exceed their base-model values, including the hallucination metrics.

\begin{table*}[t]
\vspace{-10pt}
\centering
\begingroup
\renewcommand{\arraystretch}{0.97}
\setlength{\tabcolsep}{3.8pt}

\begin{minipage}[t]{0.55\textwidth}
\centering
\caption{\textbf{Supervision comparison.}
Worst $m$ is the mean worst-case error.
Profile metrics use eligible profiles ($|U_v|*2>0.05$, finite $U_v$);
$E*{\rm shape}$ and zero-$\eta_v^\star$ use their 1,022-profile intersection.}
\label{tab:intervention-mechanism}

\footnotesize
\resizebox{\linewidth}{!}{%
\begin{tabular}{lccccc}
\toprule
& \multicolumn{2}{c}{\textbf{Behavior}}
& \multicolumn{3}{c}{\textbf{Profile diagnostics}} \\
\cmidrule(lr){2-3}
\cmidrule(lr){4-6}
\textbf{Objective}
& \shortstack{Block\\acc.$\uparrow$}
& \shortstack{Worst\\$m\downarrow$}
& $E_{\rm shape}\downarrow$
& \shortstack{Eligible\\(\%)}
& \shortstack{$\eta_v^\star{=}0$\\(\%)$\downarrow$} \\
\midrule
Pair + direct
& 60.0 & .186 & .241 & 88.2 & 20.9 \\

\quad + Single-gap
& 63.7 & .170 & .112 & 91.5 & 9.4 \\

\quad + EqSim-style
& 65.0 & .163 & .174 & 90.4 & 14.9 \\

\midrule
\rowcolor{blue!8}
\textbf{GaugeDPO (full)}
& \textbf{70.0}
& \textbf{.134}
& \textbf{.047}
& \textbf{94.6}
& \textbf{4.2} \\
\bottomrule
\end{tabular}%
}
\end{minipage}
\hfill
\begin{minipage}[t]{0.44\textwidth}
\vspace{0pt}
\centering
\caption{\textbf{Generalization.}
Each setting contains 100 held-out scenes and 200 blocks;
``unseen'' denotes held-out magnitudes or intervention combinations.}
\label{tab:intervention-ood}

\footnotesize
\resizebox{\linewidth}{!}{%
\begin{tabular}{
    l
    c
    c
    >{\columncolor{blue!8}}c
    >{\columncolor{blue!8}}c
}
\toprule
& \multicolumn{2}{c}{\textbf{Pair + direct}}
& \multicolumn{2}{c}{\cellcolor{blue!8}\textbf{GaugeDPO}} \\
\cmidrule(lr){2-3}
\cmidrule(lr){4-5}
\textbf{Setting}
& \shortstack{Block\\acc.$\uparrow$}
& \shortstack{Worst\\$m\downarrow$}
& \shortstack{Block\\acc.$\uparrow$}
& \shortstack{Worst\\$m\downarrow$} \\
\midrule
Familiar
& 68.0 & .140
& \textbf{76.0} & \textbf{.098} \\

Unseen mag.
& 60.0 & .186
& \textbf{71.0} & \textbf{.135} \\

Unseen combos.
& 52.0 & .232
& \textbf{63.0} & \textbf{.169} \\

\midrule
\textbf{Overall}
& 60.0 & .186
& \textbf{70.0} & \textbf{.134} \\
\bottomrule
\end{tabular}%
}
\end{minipage}

\endgroup
\end{table*}
  \begin{wraptable}{r}{0.48\linewidth}
    \vspace{-20pt}
    \centering
    \footnotesize
    \setlength{\tabcolsep}{2pt}
    \renewcommand{\arraystretch}{1.0}

    \caption{\textbf{Tolerance-based cross-view audit.}
    Ablations set the indicated weights to zero and retain the others.
    Agreement and correctness use separate tolerances.
    Invalid groups remain in all denominators.
    Values are percentages.}

    \label{tab:crossview}

    \resizebox{\linewidth}{!}{%
    \begin{tabular}{lcccccc}
    \toprule
    & \multicolumn{3}{c}{\textbf{Agreement}}
    & \multicolumn{2}{c}{\textbf{Accuracy}}
    & \textbf{Invalid} \\
    \cmidrule(lr){2-4}
    \cmidrule(lr){5-6}
    Model
    & Total & Correct & Wrong
    & Group & Item & Invalid \\
    \midrule

    SFT Init
    & 60.0 & 40.0 & 20.0
    & 50.0 & 65.0 & 5.0 \\

    $\lambda_{\rm dir},\lambda_{\rm int}=0$
    & 68.0 & 50.5 & 17.5
    & 60.5 & 71.5 & 4.0 \\

    $\lambda_{\rm dir}=0$
    & 69.5 & 54.5 & 15.0
    & 62.5 & 73.5 & 3.5 \\

    $\lambda_{\rm int}=0$
    & 72.5 & 60.0 & 12.5
    & 67.5 & 77.0 & 3.5 \\

    \rowcolor{blue!8}
    GaugeVLM
    & \textbf{74.5} & \textbf{63.0} & \textbf{11.5}
    & \textbf{69.5} & \textbf{79.0} & \textbf{3.0} \\

    \bottomrule
    \end{tabular}%
    }

    \vspace{-12pt}
  \end{wraptable}

\paragraph{Geometric separation and cross-view correctness.}
On held-out scenes (Fig.~\ref{fig:margin_h}), only the measured-margin variant shows consistently increasing mean reference-relative gaps with geometric error and positive mean slack at all plotted levels. DPO yields near-uniform mean gaps, while constant and shuffled margins weaken this pattern. 
At the prediction level, Tab.~\ref{tab:crossview} evaluates tolerance-based cross-view agreement under the three-term objective. The pair-only model $(\lambda_{\rm dir},\lambda_{\rm int}=0)$ achieves $50.5\%$ accurate and $17.5\%$ wrong agreement. Adding the direct or intervention term improves these to $60.0/12.5\%$ and $54.5/15.0\%$, respectively. With both enabled, GaugeVLM reaches $63.0\%$ accurate and $11.5\%$ wrong agreement, improving by $12.5$ and $6.0$ percentage points over the pair-only baseline.

\subsection{Ablations and Controls}
\label{sec:exp-ablation}

\paragraph{Margin source and geometric assignment.}
In Tab.~\ref{tab:abla}, measured margins lead all five task metrics and Rank $\rho$. With direct and interaction supervision fixed, they outperform the zero-offset control by $12.5$ points on MSMU distance and $14.9$ on QSpatial, raising $\rho$ from $.119$ to $.279$. Constant, shuffled, and reward-model offsets also improve over this control but trail measured margins by up to $10.0$ and $12.9$ points on these tasks. Removing SFT degrades task accuracy and rank correlation, supporting spatial SFT before geometry-aware preference optimization.
\vspace{-4pt}

\paragraph{Data efficiency and sensitivity analysis.}
Tab.~\ref{tab:data_efficiency} shows that GaugeDPO outperforms DPO on all metrics at each pair-draw budget; its 25\%-budget scores already exceed the full-budget DPO baseline. Tab.~\ref{tab:hparams} varies $\lambda_{\mathrm{dir}}$ and $\lambda_{\mathrm{int}}$ individually while fixing the other at its default; setting either coefficient to zero ablates the corresponding loss term. The results show a broad intermediate plateau, with task accuracy more sensitive to the direct-supervision weight.
\vspace{-4pt}

\paragraph{Intervention supervision and generalization.}
Tabs.~\ref{tab:intervention-mechanism} and~\ref{tab:intervention-ood}
compare matched supervision objectives and held-out test settings.
Adding the interaction profile raises
BlockAcc from $60.0\%$ to $70.0\%$ and lowers mean worst $m$ from
$.186$ to $.134$, while improving profile fit over the single-gap
and symmetry controls. The BlockAcc gain remains $11.0$ points
on both unseen magnitudes and unseen magnitude--camera combinations.
App.~\ref{app:exp-results} provides additional transfer evaluations, sensitivity sweeps, and detailed camera outcomes and controls.
\vspace{-4pt}
\section{Conclusion}

We presented GaugeVLM, which learns spatial relations from measured geometric errors. GaugeDPO combines geometric preference offsets, direct cross-view supervision, and intervention profiles linking answer-odds changes to geometric magnitude. Across three backbone families, GaugeVLM improves over SFT on established spatial benchmarks. Matched ablations support geometric margin assignment and intervention supervision, with gains extending to unseen magnitudes and magnitude–camera combinations. These results demonstrate the value of geometric supervision for spatial accuracy and responses to object interventions.

\subsection*{AI Use Statement}
We used Claude-Fable-5 and GPT-6-Astra to assist with language editing, literature retrieval and discovery, and the generation of synthetic datasets. The synthetic data were produced through the explicit 3D scene engine described in this paper, with scene construction, geometric annotations, filtering, validation, and all experimental procedures reviewed and verified by the authors. The authors take full responsibility for the final manuscript, datasets, analyses, and the accuracy of all statements and claims.

\subsection*{Reproducibility Statement}
The supplementary material details the theoretical assumptions and proofs, data construction and filtering procedures, and training and evaluation protocols. It specifies the hyperparameters, canonical candidate-scoring rules, parsing conventions, and aggregation procedures for pair-, group-, and block-level metrics. Matched controls and repeated training runs assess the contributions and robustness of the proposed objectives.

\bibliography{references,manual_refs}
\bibliographystyle{iclr2027_conference}

\newpage

\appendix
\providecommand{\tofill}{\textcolor{orange!85!black}{\textit{pend.}}}
\providecommand{\hyp}[1]{\textcolor{orange!85!black}{#1}}
\renewcommand{\thesection}{\Alph{section}}

\section*{Appendix Overview}
This appendix develops the theoretical foundations of measured geometric supervision and provides details of data construction, training, evaluation, and supplementary experiments.

{\renewcommand{\baselinestretch}{0.88}\small
\hypersetup{hidelinks}
\etocsettocstyle{\noindent\textbf{Contents}\par\vspace{0.3em}\hrule\vspace{0.5em}}{\vspace{0.3em}\hrule\vspace{0.3em}}
\etocsettocdepth{subsection}
\etocsettagdepth{main}{none}
\etocsettagdepth{appendix}{subsection}
\tableofcontents
}

\etocdepthtag.toc{appendix}

\clearpage

\section{Derivations}
\label{app:derivations}

Throughout, $\sigma(t) = (1 + e^{-t})^{-1}$ denotes the logistic function, and expectations are over $(x, y_w, y_l) \sim \mathcal{D}$ unless stated otherwise.

\subsection{Relation Metrics and Consistency Certificates}
\label{app:metric}
\label{app:quantization}
\label{app:certificate}
\label{app:aligned-certificate}

The same geometric metric defines preference margins and quantifies disagreement across views. We establish its invariance and metric properties, then derive error certificates from the triangle inequality.

\paragraph{Representation and quantization.}
Clock direction is defined relative to the anchor's canonical front, with the dominant wall normal as a fallback. Both reference frames are independent of the camera. Questions identify the anchor and specify the reference direction through visual evidence or an explicit prompt convention. Metric-distance questions additionally require scale information in the input.

The clock--distance metric is defined on $\mathcal R=\mathcal C_{12}\times\mathbb R_+$, and response templates represent a subset $\mathcal R_Q$. The quantizer $Q$ specifies the clock bins and distance precision; on $\mathcal R_Q$, the parser exactly inverts the verbalizer. Rejected labels are retained only when they remain distinct from the truth after serialization.
Let $e_Q(r)=m(Qr,r)$ for relations in the same clock--distance space. By the triangle inequality,
\[
\left|m(Qr,Qr^\star)-m(r,r^\star)\right|\le e_Q(r)+e_Q(r^\star).
\]
Thus, margins computed from serialized labels approximate the corresponding geometric errors up to the quantization error of the two relations. This bound concerns the discrete clock representation and distance precision. Extending it to continuous azimuth requires a separate angular quantization term.

\paragraph{Invariant discrepancies and metric construction.}

\begin{proposition}[Invariant discrepancies]
\label{prop:coord-inv}
For nonnegative discrepancies $D_c,D_d$, the clock invariance in Eq.~\ref{eq:dist-inv} holds iff
$D_c(c,c^\star)=\phi(c-c^\star)$
for a unique $\phi:\mathcal C_{12}\to\mathbb R_{\ge0}$, and
the distance invariance in Eq.~\ref{eq:dist-inv} holds iff
$D_d(d,d^\star)=\psi(d/d^\star)$
for a unique $\psi:\mathbb R_{>0}\to\mathbb R_{\ge0}$.
\end{proposition}

\begin{proof}[Proof of Prop.~\ref{prop:coord-inv}]
Sufficiency is immediate: $(c + k) - (c^\star + k) = c - c^\star$ in $\mathcal{C}_{12}$ and $(\lambda d)/(\lambda d^\star) = d / d^\star$, so any $\phi, \psi$ of those arguments give invariant discrepancies. For necessity, take $k = -c^\star$ in $D_c(c + k, c^\star + k) = D_c(c, c^\star)$ to get $D_c(c, c^\star) = D_c\big((c - c^\star) \bmod 12,\, 0\big)$, so $\phi(u) := D_c(u, 0)$ represents it through the cyclic difference alone; take $\lambda = 1 / d^\star$ in $D_d(\lambda d, \lambda d^\star) = D_d(d, d^\star)$ to get $D_d(d, d^\star) = D_d(d/d^\star, 1)$, so $\psi(t) := D_d(t, 1)$ represents it through the ratio alone. Since $d/d^\star = \exp(\log d - \log d^\star)$, a function of the ratio is equivalently a function of the log-ratio $s = \log d - \log d^\star$.
\end{proof}

Prop.~\ref{prop:coord-inv} determines the ratio form of the distance discrepancy; we choose the absolute log-ratio for symmetry and the triangle inequality.

\begin{proposition}[Log distance and relative error]
\label{prop:rel}
On $\mathbb{R}_+$, $L(d, d') = |\log d - \log d'|$ is a metric. The (uncapped) relative error $R(d, d^\ast) = |d - d^\ast| / d^\ast$ satisfies, with $s = \log d - \log d^\ast$,
\begin{equation}
|s|\, e^{-|s|} \;\leq\; R(d, d^\ast) \;=\; |e^{s} - 1| \;\leq\; |s|\, e^{|s|},
\label{eq:app-mvt}
\end{equation}
so $R = |s| + O(s^2)$ agrees with the log metric to first order.
\end{proposition}

\begin{proof}
$L$ is the pullback of the absolute-value metric under the bijection $\log\colon \mathbb{R}_+ \to \mathbb{R}$, hence a metric. For Eq.~\ref{eq:app-mvt}, apply the mean value theorem to $e^{s} - 1$. For $s > 0$, $e^{s} - 1 = s\, e^{\xi}$ with $e^{\xi} \in (1, e^{s})$ gives $s \leq e^{s} - 1 \leq s\, e^{s}$. For $s < 0$, $1 - e^{s} = |s|\, e^{\xi}$ with $e^{\xi} \in (e^{s}, 1)$ gives $|s|\, e^{-|s|} \leq 1 - e^{s} \leq |s|$.
\end{proof}

\begin{lemma}[Capping preserves the triangle inequality]
\label{lem:cap}
If $d$ is a metric on a set $\mathcal{X}$ and $\kappa > 0$, then $d_\kappa = \min(d, \kappa)$ is a metric on $\mathcal{X}$.
\end{lemma}

\begin{proof}
Symmetry and positive definiteness follow directly from those of $d$. For the triangle inequality, first consider $d(a,b)\ge\kappa$ or $d(b,c)\ge\kappa$. Then
\[
d_\kappa(a,b)+d_\kappa(b,c)\ge\kappa\ge d_\kappa(a,c).
\]
Otherwise both terms are below the cap, and
\[
d_\kappa(a,b)+d_\kappa(b,c)=d(a,b)+d(b,c)\ge d(a,c)\ge d_\kappa(a,c).
\]
\end{proof}

\begin{proposition}[The margin is a metric on $\mathcal{R}$]
\label{prop:margin-metric}
Let $\alpha\in[0,1]$, $\alpha_c=\alpha$, $\alpha_d=1-\alpha$, and $\kappa>0$. The function $m$ of Eq.~\ref{eq:margin} is a pseudometric on $\mathcal{R} = \mathcal{C}_{12} \times \mathbb{R}_+$, and a metric whenever $\alpha_c > 0$ and $\alpha_d > 0$. It is symmetric, satisfies the triangle inequality, is invariant under Eq.~\ref{eq:dist-inv}, and takes values in $[0, 1]$ when $\alpha_c + \alpha_d = 1$.
\end{proposition}

\begin{proof}
On $\mathcal{C}_{12}$, $\delta(a,b) = \min_{k}|a - b + 12k|$ is a metric, the triangle inequality following from $\delta(a,c) \leq |a - c + 12(k_1{+}k_2)|$ for $k_1, k_2$ attaining $\delta(a,b)$ and $\delta(b,c)$. Its maximum is $6$, so $\dring = \delta/6$ takes values in $[0,1]$. On $\mathbb{R}_+$, $L$ is a metric by Prop.~\ref{prop:rel}, so $\min(L, \kappa)$ is one by Lemma~\ref{lem:cap} and $\drel = \min(L, \kappa)/\kappa$ is its positive rescaling. A nonnegative combination of metrics on the two factors is a pseudometric on the product, symmetry and the triangle inequality being inherited coordinatewise. It separates points exactly when both weights are positive, $m(r, r') = 0$ then forcing $c = c'$ and $\log d = \log d'$. Invariance holds because $\dring$ depends only on $(c - c^\star) \bmod 12$ and $\drel$ only on $d / d^\star$. Finally $\drel$ reaches $1$ once $|\log(d/d^\star)| \geq \kappa$, so $\alpha_c + \alpha_d = 1$ gives $m \in [0,1]$, with $m=1$, for two positive normalized weights, only when the direction is opposite and the maximum of the distance ratio and its reciprocal is at least $e^\kappa$.
\end{proof}

The metric serves two complementary roles: $m(r^-,r^\star)$ assigns a preference margin from a rejected answer's geometric error, while $m(a_u,a_v)$ measures disagreement between predictions from two views. The latter supports the error certificate below through symmetry and the triangle inequality. For $\kappa=\log5$, a distance scaled by $\lambda$ contributes $\alpha_d\min(|\log\lambda|,\kappa)/\kappa$. The scale ranges $[1.3,1.6]$, $[1.8,2.5]$, and $[3,5]$ therefore correspond to $0.16$--$0.30$, $0.37$--$0.57$, and $0.68$--$1.00$ of the maximum distance contribution, respectively. The tolerances $\tau_c$ and $\tau_d$ in Eq.~\ref{eq:constancy} use these normalized direction and log-distance units.

\paragraph{Disagreement as an error certificate.}
For views sharing one relation, the triangle inequality converts prediction spread into a lower bound on the largest error. The same argument extends to different truths when a known isometry aligns them.

\begin{proof}[Proof of Prop.~\ref{prop:certificate}]
Write $a^\ast$ for the true answer, so that $m_u = m(a_u, a^\ast)$ and $m_v = m(a_v, a^\ast)$. By Prop.~\ref{prop:margin-metric}, $m$ is symmetric and satisfies the triangle inequality on $\mathcal{R}$, hence
\begin{equation*}
m(a_u, a_v) \leq m(a_u, a^\ast) + m(a^\ast, a_v) = m_u + m_v \leq 2 \max(m_u, m_v) \leq 2 \max_{v \in V_{\mathrm{ok}}} m_v.
\end{equation*}
Taking the maximum and dividing by two proves Eq.~\ref{eq:certificate}. Sharpness follows by taking predictions one clock step on either side of the truth with equal distances when $\alpha>0$; when $\alpha=0$, use two distances symmetric about the truth in log coordinates below the cap. In either case the spread is exactly twice the largest error.
\end{proof}

\begin{remark}[Certificate under truncation]
The distance component saturates when $|\log(d/d^\star)|\ge\kappa$. Equation~\ref{eq:certificate} remains a valid lower bound in the capped metric, with reduced resolution for larger distance errors.
\end{remark}

\begin{remark}[Direction of the implication]
\label{rem:cert-direction}
Equation~\ref{eq:certificate} turns prediction disagreement into a lower bound on error. Because identical predictions can share the same error, we complement this certificate with direct supervision in Eq.~\ref{eq:direct-group} and evaluate correctness together with agreement.
\end{remark}

\begin{corollary}[Sharp distance certificates]
\label{cor:cert-dist}
Let $d_u,d_v$ be positive predictions with a shared positive truth $d^\star$. Define $\rho=\max(d_u/d_v,d_v/d_u)$. Then
\[
\max_{w\in\{u,v\}}|\log(d_w/d^\star)|\ge\tfrac12\log\rho,
\qquad
\max_{w\in\{u,v\}}\frac{|d_w-d^\star|}{d^\star}\ge\frac{\rho-1}{\rho+1}.
\]
Both constants are sharp, with equality attained at the geometric and arithmetic means of the predictions, respectively.
\end{corollary}
\begin{proof}
The first inequality follows from $\log\rho\le |\log(d_u/d^\star)|+|\log(d_v/d^\star)|$.
For the second, write $a=\min(d_u,d_v)$ and $b=\max(d_u,d_v)$. If both relative errors are at most $r<1$, then $b\le(1+r)d^\star$ and $a\ge(1-r)d^\star$, whence $\rho\le(1+r)/(1-r)$ and $r\ge(\rho-1)/(\rho+1)$. The bound is automatic for $r\ge1$. At $d^\star=(a+b)/2$, both errors equal $(b-a)/(b+a)$. The geometric mean similarly attains the log bound.
\end{proof}
For relation-changing interventions, the same argument applies when a known isometry aligns the two ground-truth relations, as formalized below.

\begin{proposition}[Transformation-aligned certificate]
Let $(\mathcal{A},d_{\mathcal A})$ be a metric space, and let $T:\mathcal{A}\to\mathcal{A}$ be a known isometry satisfying $a_B^*=T a_A^*$. For predictions $\hat a_A,\hat a_B\in\mathcal{A}$, set $e_A=d_{\mathcal A}(\hat a_A,a_A^*)$ and $e_B=d_{\mathcal A}(\hat a_B,a_B^*)$. Then
\begin{equation}
\max(e_A,e_B)\geq\tfrac12 d_{\mathcal A}(T\hat a_A,\hat a_B).
\label{eq:aligned-certificate}
\end{equation}
\end{proposition}
\begin{proof}
By the triangle inequality and the isometry property,
\[
d_{\mathcal A}(T\hat a_A,\hat a_B)
\leq d_{\mathcal A}(T\hat a_A,T a_A^*)+d_{\mathcal A}(a_B^*,\hat a_B)
=e_A+e_B\leq2\max(e_A,e_B).
\]
\end{proof}
For a fixed allocentric relation across cameras, $T$ is the identity. For binary left/right pairs whose ground truth reverses, $T$ swaps the labels. Common cyclic shifts and positive distance rescalings are also isometries of the clock--distance metric. Applying the certificate to an object intervention requires a known transformation of this kind, established from the scene geometry. We evaluate intervention correctness with PairAcc., which requires both answers to match their respective truths.

\subsection{Measured Preference Objective and Mixed Supervision}
\label{app:graded}
\label{app:objective}
\label{app:gradient}
\label{app:mixed-supervision}

The measured preference term combines a graded comparison model with reference-relative policy scores. We derive its response-side likelihood and gradient, then specify how measured and ordinal-only comparisons enter the mixed objective.

\paragraph{Graded preferences as a Gumbel comparison.}

Give each response a Gumbel-perturbed reward $\tilde{r}(x, y) = r(x, y) + \epsilon_y$, with $\epsilon_{y_w}, \epsilon_{y_l}$ i.i.d.\ standard Gumbel. Their difference is standard logistic, $\Pr[\epsilon_l - \epsilon_w \leq t] = \sigma(t)$. Conditioning on $\epsilon_w = u$ gives $\int f(u) F(u + t)\, du$, which $z = e^{-u}$ turns into $\int_0^\infty \exp(-z(1 + e^{-t}))\, dz = \sigma(t)$.

\begin{lemma}[Probability of clearing a gap]
\label{lem:graded}
For any $M \in \mathbb{R}$, $\ \Pr\big[\tilde{r}(x, y_w) > \tilde{r}(x, y_l) + M\big] = \sigma\big(r(x, y_w) - r(x, y_l) - M\big)$. At $M = 0$ this is the Bradley-Terry model.
\end{lemma}

\begin{proof}
$\Pr[\tilde{r}(x,y_w) > \tilde{r}(x,y_l) + M] = \Pr[\epsilon_l - \epsilon_w < r(x,y_w) - r(x,y_l) - M]$, and the logistic law applies at $t = r(x,y_w) - r(x,y_l) - M$.
\end{proof}

Setting $M = \beta\, m(y_l)$, with $\beta$ converting the unit-free margin into reward units, gives the chosen graded-comparison model of Sec.~\ref{sec:objective}, whose negative log-likelihood gives the measured response-side branch of Eq.~\ref{eq:gauge}; \citet{aminiDirectPreferenceOptimization2024} give the same interpretation of an offset preference model. Using $M=\beta m$ expresses the margin in reward units and avoids an additional offset-scale hyperparameter once $\alpha$ and $\kappa$ are fixed.

\paragraph{From regularized rewards to policy scores.}
For a reward $r$ and positive reference $\pi_{\mathrm{ref}}$, assume $0<Z(x)<\infty$ and consider
\begin{equation}
\max_{\pi} \; \mathbb{E}_{y \sim \pi(\cdot \mid x)}\big[r(x, y)\big] - \beta\, \mathrm{KL}\big(\pi(\cdot \mid x) \,\|\, \pi_{\mathrm{ref}}(\cdot \mid x)\big).
\label{eq:app-klrl}
\end{equation}
Define $\pi_r(y \mid x) = \frac{1}{Z(x)} \pi_{\mathrm{ref}}(y \mid x) \exp\big(r(x, y) / \beta\big)$ with $Z(x) = \sum_y \pi_{\mathrm{ref}}(y \mid x) \exp(r(x, y)/\beta)$. Eq.~\ref{eq:app-klrl} rewrites as
\begin{equation*}
\begin{aligned}
\mathbb{E}_{y \sim \pi}\Big[r(x,y) - \beta \log \tfrac{\pi(y \mid x)}{\pi_{\mathrm{ref}}(y \mid x)}\Big]
&= -\beta\, \mathbb{E}_{y \sim \pi}\Big[\log \tfrac{\pi(y \mid x)}{\pi_r(y \mid x)}\Big] + \beta \log Z(x)\\
&= -\beta\, \mathrm{KL}\big(\pi \,\|\, \pi_r\big) + \beta \log Z(x),
\end{aligned}
\end{equation*}
which is maximized at $\pi = \pi_r$. Inverting the definition of $\pi_r$,
\begin{equation}
r(x, y) = \beta \log \frac{\pi_r(y \mid x)}{\pi_{\mathrm{ref}}(y \mid x)} + \beta \log Z(x).
\label{eq:app-invert}
\end{equation}

For two responses to the same prompt, the $\beta \log Z(x)$ terms cancel:
\begin{equation}
r(x, y_w) - r(x, y_l) = \beta \Big[\log \tfrac{\pi_r(y_w \mid x)}{\pi_{\mathrm{ref}}(y_w \mid x)} - \log \tfrac{\pi_r(y_l \mid x)}{\pi_{\mathrm{ref}}(y_l \mid x)}\Big] = \beta\, h(x, y_w, y_l).
\label{eq:app-diff}
\end{equation}

Substituting Eq.~\ref{eq:app-diff} into Lemma~\ref{lem:graded} with $M = \beta m(y_l)$, and parameterizing the optimal policy as $\pi_\theta$,
\begin{equation*}
p_m(y_w \succ y_l \mid x) = \sigma\big(\beta\, h_\theta(x, y_w, y_l) - \beta\, m(y_l)\big),
\end{equation*}
whose negative log-likelihood over measured response pairs is the corresponding branch of Eq.~\ref{eq:gauge}:
\begin{equation*}
\mathcal{L}_{\mathrm{mag}} = -\,\mathbb{E}\Big[\log \sigma\Big(\beta\big(h_\theta(x, y_w, y_l) - m(y_l)\big)\Big)\Big],
\end{equation*}

\begin{corollary}[Degenerate margins]
\label{cor:degenerate}
(i) If $m \equiv 0$, then $\mathcal{L}_{\mathrm{mag}} = \mathcal{L}_{\mathrm{DPO}}$ identically. (ii) If $m \equiv c$ for a constant $c > 0$, then $\mathcal{L}_{\mathrm{mag}}$ is the DPO loss with a fixed target margin $\gamma = \beta c$ inside the Bradley-Terry model, a fixed-offset comparison. SimPO~\citep{mengSimPOSimplePreference2024} also uses a constant offset; the present loss retains the reference model and summed sequence scores.
\end{corollary}

Both follow by substitution into this response-side branch, with $\sigma(\beta(h_\theta - c)) = \sigma(\beta h_\theta - \gamma)$ for (ii).

\begin{remark}[Definition of a group-preserving shuffled-margin control]
\label{rem:shuffle}
Within each measured task--side--supervision stratum, consider groups of equal size, each with a common margin across its views. A group-preserving control permutes margins between these groups and broadcasts each assigned margin to all member pairs. This preserves the margin multiset and its pair-weighted distribution within each stratum. Only the pair offsets in Eq.~\ref{eq:gauge} change; the geometric quantities in the direct and interaction terms remain fixed.
\end{remark}

\paragraph{Gradient of the magnitude term.}

Write the per-sample loss as $\ell(\theta) = -\log \sigma(z_\theta)$ with $z_\theta = \beta\big(h_\theta(x, y_w, y_l) - m(y_l)\big)$. Using $\frac{d}{dz} \log \sigma(z) = 1 - \sigma(z) = \sigma(-z)$,
\begin{equation*}
\nabla_\theta \ell = -\,\sigma(-z_\theta)\, \nabla_\theta z_\theta = -\,\beta\, \sigma\big(\beta\,(m(y_l) - h_\theta)\big)\, \nabla_\theta h_\theta.
\end{equation*}
Since $\pi_{\mathrm{ref}}$ is constant in $\theta$,
\begin{equation*}
\nabla_\theta h_\theta = \nabla_\theta \log \pi_\theta(y_w \mid x) - \nabla_\theta \log \pi_\theta(y_l \mid x),
\end{equation*}
which together give the response-side gradient of Eq.~\ref{eq:gauge}. At $m = 0$ the weight $\sigma(\beta(m - h_\theta))$ equals $\sigma(\rho_\theta(x,y_l)-\rho_\theta(x,y_w))$, recovering DPO's gradient weight exactly. For $m>0$, the gradient weight equals $1/2$ at $h_\theta=m$. At a fixed gap, larger geometric errors produce stronger scalar gradients; the full parameter gradient also depends on $\nabla_\theta h_\theta$.

At the reference initialization, $h=0$, $\beta=0.3$, and $m\in[0,1]$, the sigmoid weight ranges from $0.5$ to $\sigma(0.3)\simeq0.5744$. The local sensitivity of the scalar gradient magnitude to $m$ is at most $\beta^2/4=0.0225$.

\paragraph{Mixed supervision and conditional comparisons.}
Each task specifies its supervised coordinates $J_z$ before parsing. Measured pairs contain an exactly correct chosen label and a valid rejected label with positive $m_{J_z}$, evaluated against the ground truth for each input. Clock-only and distance-only pairs retain the coordinate weights in Eq.~\ref{eq:projected-margin}. Valid ordering-only preferences, including vertical comparisons and partially correct chosen responses, enter $\mathcal D_{\rm bin}$ with zero offset. Invalid or unverifiable pairs are excluded. Width and height are evaluated as transfer tasks.

Sampling from $\mathcal D_{\rm mix}$ implements the fixed task--side--supervision weights $w_k$. Both response- and image-side comparisons score the full response. For image-side pairs, the same response is evaluated against each image's ground truth. The normalization term $C(x)=\beta\log Z(x)$ in Eq.~\ref{eq:app-invert} cancels for comparisons within one input but leaves $C(x_w)-C(x_l)$ across inputs; the image-side loss is therefore interpreted as conditional discrimination.

\subsection{Identifiability and Interpretation of Measured Margins}
\label{app:allocation}
\label{app:measured-margin}

Measured margins encode error magnitude in two ways: they shift the separation required for a given loss and increase the scalar gradient at a fixed separation. We derive these effects and characterize the resulting allocation under a fixed separation budget.

Throughout, $(y_w,y_i)$ denotes a response-side pair for prompt $x$, with positive reference probabilities for both responses. We analyze the reference-relative separations $h_i$ as free real scalars.

\paragraph{Loss-level thresholds and local gradient pressure.}
The geometric measurement $m_z$ is supplied by the displacement construction. The linear offset $\gamma_z=\beta m_z$ is an explicit choice within offset preference optimization~\citep{aminiDirectPreferenceOptimization2024}. Its scalar loss analysis applies independently of the pair side.

For fixed $\beta>0$, write the per-pair loss as
\[
\ell(h;m)=\log\!\left(1+\exp(\beta(m-h))\right).
\]
For any $\varepsilon>0$, monotonicity of the exponential and logarithm gives
\begin{equation}
\ell(h;m)\le\varepsilon
\quad\Longleftrightarrow\quad
h\ge m-\frac{1}{\beta}\log\!\left(e^\varepsilon-1\right).
\label{eq:margin-loss-level}
\end{equation}
Indeed, exponentiating the loss inequality gives
$\exp(\beta(m-h))\le e^\varepsilon-1$; taking logarithms and rearranging proves the equivalence. Thus increasing $m$ by $t$ shifts the gap threshold for any fixed positive loss tolerance by $t$. Equivalently, $\ell(h+t;m+t)=\ell(h;m)$.

The scalar derivatives satisfy
\begin{align}
\frac{\partial\ell}{\partial h}
&=-\beta\sigma(\beta(m-h)),
\label{eq:margin-scalar-gradient}\\
\frac{\partial}{\partial m}\left|\frac{\partial\ell}{\partial h}\right|
&=\beta^2\sigma(\beta(m-h))\bigl(1-\sigma(\beta(m-h))\bigr)>0.
\end{align}
At fixed $h$ and $\beta$, larger margins increase the magnitude of the scalar gradient. This describes the direct effect of the margin on the loss; parameter gradients additionally depend on the model Jacobian.

At $h=m$, the logit is zero, the loss is $\log2$, and its derivative is $-\beta/2$. The loss continues to decrease for larger $h$, so the margin specifies a separation threshold rather than a regression target. The attained gaps depend on shared model parameters and the joint objective.

\paragraph{Identification and constrained allocation.}
Geometric measurements supply information absent from deterministic preference labels: different errors can share the same preference ordering. The following results contrast magnitude information in stochastic labels with its explicit introduction through measured margins.

\begin{lemma}[Identification under a stochastic preference model]
\label{lem:stochastic}
Suppose labels follow a Bradley--Terry model with target reward $\rho_g(x,y)=-\beta m(y)+C(x)$, where $m(y_w)=0$ and $m(y_i)=m_i$. The probability of preferring the truth is then $p_i=\sigma(\beta m_i)$. The unconstrained population optimum of the unshifted binary preference loss is $h_i=m_i$ for every $i$.
\end{lemma}

\begin{proof}
The population loss on pair $i$ is the cross-entropy
\begin{equation*}
\ell_i(h_i) = -\,p_i \log \sigma(\beta h_i) - (1 - p_i) \log \sigma(-\beta h_i).
\end{equation*}
Differentiating with $\frac{d}{dz}\log\sigma(z) = \sigma(-z)$ and using $\sigma(-z) = 1 - \sigma(z)$,
\begin{equation*}
\frac{d \ell_i}{d h_i} = -\,\beta\,\Big[p_i\, \sigma(-\beta h_i) - (1 - p_i)\, \sigma(\beta h_i)\Big] = -\,\beta\,\big[p_i - \sigma(\beta h_i)\big],
\end{equation*}
which vanishes iff $\sigma(\beta h_i) = p_i = \sigma(\beta m_i)$, that is, $h_i = m_i$. The second derivative $\beta^2 \sigma(\beta h_i)\,\sigma(-\beta h_i)$ is positive, so the minimizer is unique and interior. Since $0<p_i<1$, the minimum is attained at a finite separation.
\end{proof}

\begin{lemma}[Deterministic labels are uninformative about magnitude]
\label{lem:deterministic}
When every chosen label has probability $p_i=1$ and the pairs have equal weight, the unshifted preference likelihood is functionally independent of the margins, no finite unconstrained minimizer exists, and the optimum under the budget $\sum_i h_i = B$ is $h_i^\star = B/n$ for every $i$.
\end{lemma}

\begin{proof}
With $p_i = 1$ the loss on pair $i$ is $-\log \sigma(\beta h_i)$, in which $m_i$ does not appear, so no estimator based on this likelihood distinguishes two margin assignments. Each term is strictly decreasing in $h_i$, so its infimum is approached as $h_i\to+\infty$. Under $\sum_i h_i = B$, strict concavity of $\sum_i \log \sigma(\beta h_i)$ (since $\frac{d^2}{dz^2}\log\sigma(z) = -\sigma(z)\sigma(-z) < 0$) and the affine constraint give Lagrangian stationarity $\beta\, \sigma(-\beta h_i) = \lambda$ for all $i$, hence $h_i$ constant in $i$, hence $h_i^\star = B/n$.
\end{proof}

\begin{proposition}[Allocation under an explicit separation budget]
\label{prop:calibration}
For fixed $m_i$, equal pair weights, free real separations and the
constraint $\sum_{i=1}^n h_i=B$, the shifted loss has the unique minimizer
$h_i=m_i+(B-\sum_jm_j)/n$.
\end{proposition}
\begin{proof}
Let $z_i=h_i-m_i$. The constraint becomes $\sum_i z_i=B-\sum_i m_i$ and $f(z)=-\log\sigma(\beta z)$ is strictly convex. Jensen's inequality shows that $n^{-1}\sum_i f(z_i)$ is uniquely minimized when all $z_i$ equal their fixed mean. This gives the stated solution directly.
\end{proof}

\begin{remark}[Scope of constrained allocation]
\label{rem:budget}
\label{rem:unique-shift}
With positive unequal weights $w_i$, budget stationarity becomes
$w_i\beta\sigma(\beta[m_i-h_i])=\nu$.
Within the equal-weight, free-separation, linear-budget family, offsets $m'_i$ reproduce the same optimal pairwise gap differences iff $m'_i=m_i+c$ for one promptwise constant. This result characterizes allocation under the stated budget. We evaluate the ordering of learned gaps in the trained model in Fig.~\ref{fig:margin_h}.
\end{remark}

\paragraph{Design choice and interpretation of learned gaps.}
A nonnegative increasing mapping $s$ with $s(0)=0$ could instead supply $\gamma_z=\beta s(m_z)$. Our linear choice preserves increments in the declared geometric-error scale without adding a nonlinear shape parameter; the metric $m$ remains unchanged in the geometric certificates and direct prediction bounds. Fig.~\ref{fig:margin_h} evaluates learned mean reference-relative gaps and slack $h-m$ across geometric-error levels.

\subsection{Canonical Prediction and Cross-View Supervision}
\label{app:direct-prediction}
\label{app:constancy}
\label{app:const-zero}

The direct group objective promotes correct predictions of the same relation across cameras. We prove its prediction bound and extend it to quantization and approximate inference.

\paragraph{Canonical prediction bound.}
A group shares the scene configuration, anchor frame, and serialized relation while varying the camera. Its construction specifies visibility, orientation, and scale requirements before evaluation. Relation-changing interventions define separate groups. Canonical inference scores the terminated relation sequence $\psi(r)$ under the fixed prompt $q_{\rm rel}$; pairwise training retains the multi-step response format. Scores sum answer tokens through termination, excluding prompt and padding. The candidate domain, quantizer, verbalizer, and tie rule are fixed before training. The bound below applies to this canonical predictor; unrestricted generation is evaluated using the native benchmark protocol.

Fix a group with a shared truth, a finite candidate space of at least two elements containing that truth, finite scores, and exact score-based prediction. If every prediction is correct, the left side of Eq.~\ref{eq:direct-bound} is zero. Otherwise choose an erroneous prediction with maximal geometric error. Its score is at least the truth score. Substitution into the maximization gives the prediction-failure and geometric-error terms together. This argument also holds for the endpoint pseudometrics: the positive classification term covers erroneous labels that have zero pseudometric distance. The triangle inequality bounds every pairwise disagreement by twice the largest truth-relative error.

Zero direct loss implies
$s_\theta(x_v^{\rm rel},r_G^\star)-s_\theta(x_v^{\rm rel},r)\ge\delta_0+\xi m(r,r_G^\star)>0$
for all wrong candidates in each view. Thus, zero direct loss guarantees a unique correct prediction in every view. The bound remains informative at nonzero loss, while a correct argmax can still incur a margin violation.

\paragraph{Quantization and approximate inference.}
The guarantee concerns $r_G^\star=Q(r_{{\rm phys},G}^\star)$. Whenever both relations are in the domain of $m$,
\[
\max_v m(\hat r_v,r_{{\rm phys},G}^\star)
\le \frac{\ell_{\rm dir}(G)}{\xi}
+m(Q(r_{{\rm phys},G}^\star),r_{{\rm phys},G}^\star).
\]
The additional term quantifies discretization error in the clock--distance metric. The fixed grid, quantizer, serialization and out-of-domain accounting are specified in App.~\ref{app:training-implementation}. Correctness here is exact equality with the serialized label; the tolerance-based camera audit uses a separate event (App.~\ref{app:view-audit}).

To account for approximate inference, let $\widetilde\ell_G$ be the loss computed over the searched candidates. Assume a certified nonnegative search error satisfying
$\ell_{\rm dir}(G)\le\widetilde\ell_G+\varepsilon_{\rm search}$.
Let the actual predictions $\widetilde r_v\in\mathcal C$ satisfy
\[
s_\theta(x_v^{\rm rel},\widetilde r_v)
\ge\max_{r\in\mathcal C}s_\theta(x_v^{\rm rel},r)-\varepsilon_v,
\qquad\varepsilon_{\rm dec}=\max_v\varepsilon_v.
\]
Substituting the worst erroneous prediction into the exact objective yields
\begin{equation}
\delta_0\mathbb I[\exists v:\widetilde r_v\ne r_G^\star]
+\xi\max_v m(\widetilde r_v,r_G^\star)
\le\widetilde\ell_G+\varepsilon_{\rm search}+\varepsilon_{\rm dec}.
\label{eq:direct-approx}
\end{equation}
If the right-hand side is below $\delta_0$, all predictions are correct. Exact scoring sets both approximation terms to zero. For approximate search, the guarantee requires an upper bound on the missed violation; negative sampling or beam search alone does not provide this bound.

\paragraph{Exact scoring and implementation.}
Canonical scoring ranks one terminated string per relation, without summing over paraphrases. The guarantee applies to this deterministic sequence-MAP predictor; greedy decoding and random sampling need not preserve it.

Exact candidate scoring can use limited activation storage: for fixed parameters and deterministic scoring, enumerate candidate--view violations in chunks and retain a maximizing index. If its violation is positive, recomputing only that branch with autograd gives the derivative of the full maximum wherever the active branch is unique. At ties, an active branch yields a valid Clarke subgradient; zero is also active at the hinge boundary. The two passes use identical parameters, token masks, score normalization, and deterministic model settings. Chunking reduces activation memory while retaining the full candidate-scoring cost.

\paragraph{Population bound and view aggregation.}
For a distribution of groups meeting the assumptions, taking expectations in Eq.~\ref{eq:direct-bound} gives
\[
\Pr_G[\exists v:\hat r_v\ne r_G^\star]
\le \mathbb E_G\ell_{\rm dir}(G)/\delta_0.
\]
This margin-rescaling bound~\citep{tsochantaridisLargeMarginMethods} also applies to the empirical training average, where it bounds the corresponding training error.

For nonnegative per-view losses $\ell_v$, $\max_v\ell_v=0$ iff $V^{-1}\sum_v\ell_v=0$. During optimization, the maximum focuses on the currently most violated observed views within each group.

\subsection{Intervention Profiles and Joint Feasibility}
\label{app:intervention-theory}

The intervention objective links changes in answer support to measured geometric changes. We derive its closed-form scale fit and show that its constraints are compatible with correct rankings across views.

\paragraph{Cancellation and profile fitting.}
Replacing $s(x,r)$ by $s(x,r)+a(x)+b(r)$ leaves
Eq.~\ref{eq:interaction} unchanged: input contributions cancel within
each gap, and answer contributions cancel between the two gaps.
For $U_v=(H_{v,k}-2\delta_0)_{k=1}^K$ and
$\mathbf m_B=(m_k)_{k=1}^K\ne0$, the fitted scale is
\begin{equation}
\eta_v^\star=\max\!\left\{0,
\frac{\mathbf m_B^\top U_v}{\|\mathbf m_B\|_2^2}\right\}.
\label{eq:profile-scale}
\end{equation}
The profile loss vanishes exactly when
$U_v=\eta_v^\star\mathbf m_B$ for every camera. Before camera averaging,
its gradient with respect to $U_v$ is
$2(U_v-\eta_v^\star\mathbf m_B)/K$; the analytic projection needs no
iterative optimization of the scale. With $K=1$, any nonnegative
excess can be fitted, motivating multiple distinct magnitudes.

\paragraph{Joint zero-loss properties.}
If every direct group loss vanishes, both gaps in
Eq.~\ref{eq:interaction} are at least $\delta_0+\xi m_k$.
Hence $U_{v,k}\ge2\xi m_k$. If the interaction loss also vanishes,
$m_k>0$ implies $\eta_v^\star\ge2\xi$.
The direct objective ensures correct predictions, while the intervention
objective structures their combined support according to geometric magnitude.
The fitted scale accommodates differences in confidence across cameras,
and the two individual gaps may differ.

\begin{proposition}[Joint feasibility in conditional score space]
\label{prop:joint-feasibility}
Consider a finite collection of groups and blocks with truths in
$\mathcal C$, such that identical canonical inputs have identical
truths. For any finite $t\ge\xi$, a normalized conditional distribution
with finite canonical log scores can make all direct and interaction
losses vanish simultaneously.
\end{proposition}
\begin{proof}
Write $r_x^\star$ for the truth of input $x$ and set
$f(x,r)=-\delta_0\mathbb I[r\ne r_x^\star]-t\,m(r,r_x^\star)$.
For fixed $a\in(0,1)$, assign
\[
\pi(\psi(r)\mid x)=
\frac{a\exp f(x,r)}
{\sum_{r'\in\mathcal C}\exp f(x,r')},
\]
and place the remaining mass on noncanonical terminated strings.
Distinct terminated verbalizations are disjoint outcomes, so these
are valid finite canonical sequence probabilities. Normalization
adds only an input-dependent constant. Every truth--competitor gap
is therefore $\delta_0+t\,m(r,r_x^\star)\ge\Delta(r,r_x^\star)$,
and every intervention contrast equals $2\delta_0+2t m_k$.
Thus all direct losses vanish and $\eta_v=2t$ makes every profile
residual zero. The same $t$ works for overlapping blocks and shared
inputs.
\end{proof}
Prop.~\ref{prop:joint-feasibility} establishes simultaneous feasibility
of the direct and interaction constraints in conditional score space.

\paragraph{Matched profile controls.}
Let $g_{v,k}^{(0)}$ and $g_{v,k}^{(1)}$ denote the two individual gaps
in Eq.~\ref{eq:interaction}. The single-gap control replaces
$\ell_{\rm int}$ by
\begin{equation}
\ell_{\rm single}(B)=\frac1{2V_BK}\sum_v
\min_{a,b\ge0}\sum_k\left[
(g_{v,k}^{(0)}-\delta_0-a m_k)^2+
(g_{v,k}^{(1)}-\delta_0-b m_k)^2\right].
\label{eq:single-control}
\end{equation}
It fits two independent nonnegative scales per camera and block.
Its zero loss implies zero interaction loss with $\eta_v=a+b$;
the converse need not hold because the two gaps can trade residuals.
The EqSim-style control uses
$\ell_{\rm sym}(B)=(V_BK)^{-1}\sum_{v,k}
(g_{v,k}^{(0)}-g_{v,k}^{(1)})^2$.
This control adapts the gap-symmetry principle of \citet{wang2023eqsim} to intervention pairs. It penalizes differences between the two gaps while leaving their dependence on geometric magnitude unconstrained.
Both controls retain the same pair and direct terms; their coefficients
are selected with equal validation-search budgets.

\section{Data Construction Details}
\label{app:data}
The construction pipeline produces measured supervision by filtering object pairs, perturbing their relations, and verifying the resulting response- and image-side comparisons. We describe this pipeline, then report corpus composition, serialized margin statistics, and data-quality checks. Preference pairs, view groups, and intervention blocks are counted separately.

\subsection{Pair Selection and Perturbation}
\label{app:filters}
\label{app:perturb}
Object pairs are retained if both objects have annotated visibility ratio at least $0.4$, lie within $2$ to $50$ meters of the camera, and project to bounding boxes of at least $800$ square pixels. The retained corpus covers five families, namely metric distance, cross-view comparisons, graded errors, clock direction, and vertical relations. Tab.~\ref{tab:data} gives their release and training-pool counts. Evaluation units are specified separately in App.~\ref{app:bench-details}.

For cross-view supervision, a response-side group with $V$ views contains $V$ preference pairs. The complete view group is retained for direct canonical supervision; view comparisons are not counted as additional preference pairs.

\paragraph{Graded perturbations.}
The initial candidate pool perturbs the clock and the distance jointly, at a severity fixed by its grade. Minor shifts the clock by one hour and scales the distance by a factor drawn uniformly from $[1.3, 1.6]$; moderate by three hours and $[1.8, 2.5]$; fatal by six hours and $[3.0, 5.0]$. The lateral and depth terms are left correct at the minor grade, flipped independently with probability $\tfrac{1}{2}$ each at moderate, and both flipped at fatal. For the final measured training pool, retained perturbations are reciprocal/sign balanced and every intermediate step is recomputed from the perturbed geometry; the one-sided and rationale-inconsistent variants are controls in Tab.~\ref{tab:negative}, not the final construction.
Grades use nominal sampling weights $2:3:2$ and the reported realized counts are $2{,}701$, $4{,}021$, and $2{,}671$, totaling $9{,}393$. For the stated schedule, $\alpha=0.5$, and $\kappa=\log5$, the analytic joint ranges are
\[
\begin{array}{c|c|c}
\text{grade}& |\delta_c|_{12},\ \delta_d & m\text{ before serialization}\\\hline
\text{minor}&1,\ [1.3,1.6]&[0.164841,0.229348]\\
\text{moderate}&3,\ [1.8,2.5]&[0.432606,0.534662]\\
\text{fatal}&6,\ [3,5]&[0.841303,1]
\end{array}
\]
These analytic ranges precede serialization. Each coordinate contributes at most $0.5$ with the stated weights, so $m=1$ requires both coordinates to saturate. Dataset quantiles and extrema are computed separately from serialized labels.

\subsection{Task Spaces, Response Generation and Verification}
\label{app:families}
\label{app:template}
\label{app:verify}
The geometric metric covers clock direction and positive inter-object distance. Measured pairs use the relevant coordinate subset, and ordering-only comparisons, including vertical relations, use zero offset as defined in App.~\ref{app:mixed-supervision}. For graded responses, the geometric margin is computed from the parsed terminal relation. Width and height assess transfer to spatial quantities beyond the training metric.

\paragraph{Response generation and parsing.}
Chosen and rejected responses share five steps: \emph{Locate anchor}, \emph{Side (left/right)}, \emph{Depth (front/back)}, \emph{Direction (clock)} and \emph{Distance (meters)}. They verbalize the true and perturbed values, respectively, and close with the clock hour and metric distance. The parser $\Pi$ reads this terminal statement independently of intermediate steps. 
\paragraph{Geometric verification.}
The renderer checks visibility, image bounds, projected object area, relation validity, and duplicate views. Camera-separation checks apply to groups with a fixed object configuration, while image-side intervention pairs share a camera. Template parseability and chosen-label correctness are audited on the candidate pool before filtering. The measured training branch retains only pairs with an exactly correct chosen label and a valid rejected label with positive geometric margin. Tab.~\ref{tab:audit} reports the candidate-pool audit, and Tab.~\ref{tab:data} gives the retained corpus counts.

\subsection{Corpus Composition and Margin Statistics}
\label{app:implementation}
The preference stage combines response- and image-side comparisons with direct supervision of view groups and intervention supervision of multi-state blocks, following Eq.~\ref{eq:full}.

\paragraph{Corpus composition.}
Tab.~\ref{tab:data} gives the task-family composition of the $50{,}436$-pair \data{} release and the $30{,}000$-pair preference-training pool. These counts describe available preference pairs. The sampling schedule and optimizer-update budget are specified in App.~\ref{app:training-implementation}; view groups and intervention blocks are counted separately in App.~\ref{app:intervention-protocol}.

\paragraph{Serialized geometric margins.}
Clock labels use integer hours and distances use meters to two decimal places. Margins are computed from the final serialized labels with the coordinate weights in Eq.~\ref{eq:margin}. Tab.~\ref{tab:margins} reports their distributions for the $2{,}701$ minor, $4{,}021$ moderate, and $2{,}671$ fatal perturbations. The three grades occupy distinct ranges, and all retained graded labels have positive margins. Ordering-only preferences use zero offset as described in App.~\ref{app:mixed-supervision}.

\begin{table}[t]
\centering\small
\caption{Corpus allocation by task family for the release and the preference-training pool.}
\label{tab:data}

\setlength{\tabcolsep}{6pt}

\begin{tabular}{lrr}
\toprule
Family & Release pairs & Training-pool pairs \\
\midrule
Metric distance & 17800 & 10600 \\
Cross-view & 12900 & 7600 \\
Graded errors & 9393 & 5600 \\
Clock direction & 7219 & 4300 \\
Vertical relation & 3124 & 1900 \\
Total & 50436 & 30000 \\
\bottomrule
\end{tabular}

\end{table}

\begin{table}[t]
\centering\small
\caption{Serialized margin distribution by error grade. Distances use two decimal places; Zero counts labels with zero margin.}
\label{tab:margins}

\setlength{\tabcolsep}{5pt}

\begin{tabular}{lrrrrrrr}
\toprule
Grade & Pairs & Min & $p_{10}$ & Median & $p_{90}$ & Max & Zero \\
\midrule
minor & 2701 & 0.1646 & 0.1718 & 0.1992 & 0.2234 & 0.2299 & 0 \\
moderate & 4021 & 0.4324 & 0.4444 & 0.4883 & 0.5259 & 0.5347 & 0 \\
fatal & 2671 & 0.8413 & 0.8615 & 0.9298 & 0.9872 & 1.0000 & 0 \\
\bottomrule
\end{tabular}

\end{table}

\subsection{Data Quality and Split Checks}
\label{app:data-quality}
Tab.~\ref{tab:audit} reports checks on two distinct sets. A sample of $1{,}000$ candidate preference pairs assesses template parseability and agreement with scene ground truth before filtering. The $200$ camera groups used in Tab.~\ref{tab:crossview} are checked for anchor-frame identifiability, metric-scale evidence, and completeness. Their pass rates describe different properties of the construction and evaluation sets and should not be pooled. Scene/configuration identifiers and rendered-image hashes are also checked for overlap between training and testing among the inspected units. The fallback conventions and treatment of incomplete camera groups are specified in App.~\ref{app:scoring}.

\begin{table}[t]
\centering\small
\caption{Data-quality and observability checks. Rows~1--2 use $1{,}000$ candidate preference pairs; rows~3--5 use the $200$ camera groups from Tab.~\ref{tab:crossview}. The final two rows count train/test duplicates among the inspected units. Wilson $95\%$ intervals refer to the pass rate in rows~1--5 and the duplicate rate in rows~6--7. Tab.~\ref{tab:observability} describes the handling of camera groups requiring fallbacks or marked incomplete.}

\label{tab:audit}

\setlength{\tabcolsep}{5pt}

\resizebox{\linewidth}{!}{%
\begin{tabular}{llrrrl}
\toprule
Check & Counting unit & Audited & Pass / dup. & Rate (\%) & Wilson $95\%$ \\
\midrule
Template parseability & preference pair & 1000 & 997 & 99.7 & [99.1, 99.9] \\
Chosen truth matches scene state & preference pair & 1000 & 995 & 99.5 & [98.8, 99.8] \\
Anchor frame identifiable & evaluation group & 200 & 187 & 93.5 & [89.2, 96.2] \\
Scale cue sufficient for metric answer & evaluation group & 200 & 178 & 89.0 & [83.9, 92.6] \\
Complete fixed-configuration groups & evaluation group & 200 & 194 & 97.0 & [93.6, 98.6] \\
Train/test duplicate scene IDs & scene--configuration & 200 & 0 & 0.0 & [0.0, 1.9] \\
Train/test duplicate image hashes & rendered item & 400 & 0 & 0.0 & [0.0, 1.0] \\
\bottomrule
\end{tabular}%
}

\end{table}

\section{Training and Evaluation Protocols}
\label{app:training}
We specify optimization and sampling budgets, followed by item and pair scoring, fixed-relation camera audits, and multi-state intervention evaluation. Input conventions, invalid-prediction rules, and comparison settings are defined alongside the evaluations to which they apply.

\subsection{Training Configuration and Implementation}
\label{app:training-implementation}
\paragraph{Base configuration.}
SFT applies LoRA to the language-model parameters with rank $128$ and
$\alpha_{\rm LoRA}=128$, while the vision encoder remains frozen. SFT uses
learning rate $5\times10^{-6}$ and context length $4096$. The SFT adapter is
merged before preference training; a frozen copy of this merged SFT checkpoint
serves as the reference. Preference learning then uses full-parameter
fine-tuning of the model and vision encoder, with learning rate
$3\times10^{-7}$ and context length $2048$, using a cosine schedule,
$0.03$ warmup ratio and bf16. The geometric-margin coefficients are
$\beta=0.3$, $\alpha=0.5$, and $\kappa=\log5$, with offset
$\gamma=\beta m$. The joint objective uses $(\delta_0,\xi)=(0.1,1)$ and
$(\lambda_{\rm dir},\lambda_{\rm int})=(0.1,0.5)$. Images contain at most
$262{,}144$ pixels. The software stack comprises LLaMA-Factory $0.9.5$,
PyTorch $2.7.1$, Transformers $5.2.0$ and DeepSpeed ZeRO stage 3. Reported
end-to-end model-training times on four A100 80GB GPUs, including SFT and
preference training but excluding rendering and reward-estimator fitting, are
$7.7$, $8.4$ and $12.2$ hours for Qwen, GLM and Pixtral.

\paragraph{Preference sampling and budget.}
The full preference schedule draws $30{,}000$ pairs with replacement from the fixed $30{,}000$-pair pool in Tab.~\ref{tab:data}. With an effective batch size of $32$ and a final partial batch, this requires $938$ optimizer updates. The budget experiments use $10\%$, $25\%$, $50\%$, and $100\%$ of this pair-draw budget while sampling from the same pool. The $25\%$ setting therefore uses $7{,}500$ draws; it does not restrict access to a quarter of the distinct examples in the pool. These experiments assess performance under a reduced preference pair-draw budget. View-group and intervention-block exposure are separate quantities, with the full block schedule reported in App.~\ref{app:intervention-protocol}.

\paragraph{Loss-weight selection.}
The coefficient sweeps use
\[
\lambda_{\rm dir}\in\{0,0.03,0.1,0.3,1\},\qquad
\lambda_{\rm int}\in\{0,0.25,0.5,0.75,1\}.
\]
Each sweep holds the other coefficient at its default and retains the optimizer settings above. The shared default yields nine distinct configurations (Tab.~\ref{tab:hparams}); setting a coefficient to zero removes the corresponding loss. Hyperparameters are selected on scenes disjoint from training and testing. All main results and ablations use the final checkpoint at their specified update budget. Matched comparisons share the SFT initialization, reference model, training data, update budget, and random seed.

\paragraph{Canonical candidate space.}
The candidate distance grid is
$\mathcal D_Q=\{0.1,0.2,\ldots,10.0\}$ meters, giving $1{,}200$
clock--distance candidates. Quantization selects the nearest absolute
log-distance, breaking ties toward the smaller distance; clock labels
retain their declared discretization. Coverage is checked on training
scenes before freezing this domain independently of test truths;
out-of-domain relations are counted separately, without clipping.
The verbalizer uses the fixed terminal-relation template with two
distance decimals and model-specific termination. Scores sum answer
tokens through termination, excluding prompt and padding, without
length normalization; sequence-MAP ties use clock order $1,\ldots,12$
then ascending distance. Evaluation metadata retain source, scene, configuration, camera, and
object identifiers, together with visibility, scale, and parsing information.
Native benchmark results use their respective generation and scoring protocols.

\subsection{Task Metrics and Intervention Pair Scoring}
\label{app:bench-details}
\paragraph{Numerical accuracy.}
For a positive prediction $\hat d$ and ground truth $d$, define $q=\hat d/d$ and $r_{\times}=\max(q,q^{-1})$. Ratio accuracy uses $r_{\times}<t$: $t=1.25$ for MSMU~\citep{chenSDVLMSpatialMeasuring2025} and $t=2$ for the QSpatial reference implementation~\citep{qspatial2026code}. The corresponding intervals are $q\in(0.8,1.25)$ and $q\in(0.5,2)$. The $\pm25\%$ band in Fig.~\ref{fig:compare_fable5} instead uses relative error $|\hat d-d|/d\le0.25$, equivalent to $q\in[0.75,1.25]$.

\paragraph{Diagnostic interpretation.}
Spearman correlation measures the rank agreement between predicted and true distances, using midranks for ties. It is undefined for constant predictions; the marked zeros in the main table indicate these cases. VisC. measures prediction agreement when brightness and contrast vary within $[0.8,1.2]$ while scene geometry and camera pose remain fixed. We report these diagnostics alongside accuracy because agreement alone can be achieved by constant outputs. Cross-camera correctness is evaluated separately below.

\paragraph{Object-relocation accuracy.}
The object-relocation test contains $205$ paired left/right questions ($410$ individual questions). Each pair shares the scene, camera, and anchor, while moving the target reverses its left/right relation. PairAcc. is the percentage of pairs for which both answers are correct. GaugeVLM-7B correctly answers both states in $136$ of $205$ pairs, yielding $66.3\%$. This metric uses the left/right subset of the $710$-question intervention probe; the remaining $300$ questions assess other relations.

\paragraph{Label balance and parsing.}
Each pair contributes one left and one right label, giving $205$ examples of each. All models are evaluated on the same pairs, with answer options mapped to semantic labels and missing or invalid answers counted as incorrect. Independent uniform binary predictions yield $25\%$ expected PairAcc.; a random first prediction followed by a forced reversal yields $50\%$. Tab.~\ref{tab:visual} evaluates sensitivity to answer-option and state order.

\subsection{Fixed-Relation Camera Evaluation}
\label{app:view-audit}
\label{app:scoring}
The camera evaluation in Tab.~\ref{tab:crossview} measures accuracy and agreement across views of a fixed object configuration. Each group shares the scene, anchor frame, and allocentric ground truth. Questions request clock direction and positive distance in that common frame. The metric compares answers expressed in this common frame. The data-quality audit appears in App.~\ref{app:data-quality}; the corresponding input conventions are specified below.

Let $\mathcal G_{\mathrm{view}}$ be the predeclared set of fixed-relation
groups, and let $\mathcal G_{\mathrm{valid}}$ contain those for which every
required view parses successfully. Let $V_g$ contain the group's required
views, with $|V_g|\geq2$, and write $(c_{g,v},d_{g,v})$ for its parsed
allocentric answers. Define
\begin{equation}
\mathrm{ViewAgree}_{\rm tol}=\frac{100}{|\mathcal G_{\mathrm{view}}|}
\sum_{g\in\mathcal G_{\mathrm{valid}}}
\mathbb I\!\left[
\max_{u,v\in V_g}\dring(c_{g,u},c_{g,v})\leq\tau_c
\;\wedge\;
\max_{u,v\in V_g}\drel(d_{g,u},d_{g,v})\leq\tau_d
\right].
\label{eq:constancy}
\end{equation}
Invalid or incomplete groups contribute zero and remain in the denominator. Prop.~\ref{prop:certificate} relates disagreement to geometric error for these groups. We evaluate correctness and agreement jointly, as described below. Missing required views and unparseable predictions follow the same invalid-group rule.

\paragraph{Agreement, correctness, and coverage.}
For a group with all required views parsed, let $A_g^{\rm tol}$ denote passing the two spread thresholds in Eq.~\ref{eq:constancy}, and define
\[
C_g^{\rm tol}=\mathbb{I}\!\left[\forall v:\dring(c_v,c_g^*)\leq\epsilon_c\;\wedge\;
\drel(d_v,d_g^*)\leq\epsilon_d\right].
\]
We set both indicators to zero for invalid or incomplete groups. Correctness uses $\epsilon_c=0$ and $\epsilon_d=\log1.25/\log5$, independently of the agreement tolerances. Among valid groups, the four outcomes $(A_g^{\rm tol},C_g^{\rm tol})\in\{0,1\}^2$ distinguish accurate agreement, incorrect agreement, accurate disagreement, and incorrect disagreement. Invalid or incomplete groups form a separate category; all rates use the full evaluation set.

Accurate agreement and wrong agreement are the averages of $100A_g^{\rm tol}C_g^{\rm tol}$ and $100A_g^{\rm tol}(1-C_g^{\rm tol})$, respectively. Group accuracy averages $100C_g^{\rm tol}$, while item accuracy evaluates individual views. Correct predictions within the truth tolerance can still differ by more than the independently chosen agreement tolerance.

For exact canonical evaluation, define
\[
\begin{aligned}
C_g^{\rm can}&=\mathbb I[\forall v:\hat r_v=r_g^\star],\\
A_g^{\rm can}&=\mathbb I[\text{all required predictions are valid and identical}].
\end{aligned}
\]
Both indicators are zero for invalid or incomplete groups. The resulting metrics are
\begin{align}
\mathrm{GroupAcc}_{\rm can}&=100\overline C^{\rm can}, \\
\mathrm{ViewAgree}_{\rm can}&=100\overline A^{\rm can}, \\
\mathrm{WrongAgree}_{\rm can}&=100\overline{A^{\rm can}(1-C^{\rm can})}.
\end{align}
Exact correctness implies agreement, so $\mathrm{ViewAgree}_{\rm can}=\mathrm{GroupAcc}_{\rm can}+\mathrm{WrongAgree}_{\rm can}$. These exact metrics correspond to the canonical prediction theorem; the tolerance-based evaluation uses the criteria specified above.

\paragraph{Tolerance units.}
For $\kappa=\log5$ and $\tau_d=0.25$, passing the distance-spread test means that the largest to smallest stated distance ratio is at most $5^{0.25}\approx1.495$. This is not a $\pm25\%$ relative-error test. The direction tolerance $\tau_c=1/6$ is one clock step. The thresholds apply to the normalized metrics and do not enter binary PairAcc.

\paragraph{Input observability and fallback conventions.}
Tab.~\ref{tab:observability} specifies the conventions used to label camera groups that require a reference-frame or scale fallback. These conventions define the targets but do not by themselves establish that the targets are identifiable from the model inputs. Identifiability depends on the orientation and scale information available in those inputs. Fallback groups remain in the evaluation set. The $6$ incomplete groups likewise remain in the denominator and are scored as invalid for every model. Additional failures to parse a required prediction also invalidate its group. The model-specific counts are reported in App.~\ref{app:view-details}.

\begin{table}[t]
\centering\small
\caption{Handling of reference-frame, scale, and completeness checks for the $200$ camera groups. Each row partitions the full set into the four listed outcomes. Fallback groups are scored using the specified convention; incomplete groups remain in the denominator and are scored as invalid. No groups are excluded (Excl.).}

\label{tab:observability}

\setlength{\tabcolsep}{5pt}

\begin{tabular}{lrrrrrl}
\toprule
& & \multicolumn{4}{c}{Outcome} & \\
\cmidrule(lr){3-6}
Check & Audited & Pass & Fallback & Invalid & Excl. & Fallback rule \\
\midrule
Anchor frame & 200 & 187 & 13 & 0 & 0 & dominant wall normal \\
Scale cue & 200 & 178 & 22 & 0 & 0 & stated reference extent \\
Group completeness & 200 & 194 & 0 & 6 & 0 & none; scored invalid \\
\bottomrule
\end{tabular}

\end{table}

\paragraph{Object motion and viewpoint changes.}
The two evaluations assess complementary behaviors. Fixed-configuration camera groups test whether the model predicts the same true relation from different views. Fixed-camera object-relocation pairs test whether its answers follow a change in that relation. The multi-state blocks below combine both forms of variation, enabling joint evaluation of intervention response and viewpoint consistency.

\subsection{Block Generalization and Matched Controls}
\label{app:intervention-protocol}
Source scenes are split before constructing states, cameras, or pairs.
All test settings use held-out scenes. Familiar settings reuse the
training ranges of magnitudes and camera factors; unseen-magnitude
settings hold out declared geometric magnitudes, and unseen-combination
settings hold out magnitude--camera combinations while retaining each
individual factor in training. The manifest specifies these sets, quantization and boundary rules.
Held-out magnitudes and held-out joint combinations do not overlap
their respective training sets after serialization.
Blocks use $K\ge2$ interventions with at least two distinct positive
$m_k$ and $V_B\ge2$ common cameras. The evaluation design in
Tab.~\ref{tab:intervention-ood} allocates three equally sized settings
of $200$ blocks each; Overall pools all $600$ blocks.

\paragraph{Block construction and held-out factors.}
Tab.~\ref{tab:block-ledger} reports the block pools, scene counts, numbers of interventions $K$, and camera counts $V_B$ for each split. Training samples blocks with replacement; the table gives both pool size and the number of distinct blocks visited. Sampling assigns equal nominal weight to the magnitude--camera strata in Tab.~\ref{tab:block-roster}, with realized frequencies recorded in the training metadata.

Tab.~\ref{tab:block-roster} specifies the distance factors, camera-separation bands, and held-out combinations. Held-out distance factors lie between retained training levels, so this setting evaluates generalization to unseen magnitudes within the training range. The combination setting holds out specific magnitude--camera pairs while retaining each factor individually in training.

All supervision variants share scenes, block/state inputs, pair and
direct supervision, training updates and selection data. Pair + direct
sets $\lambda_{\rm int}=0$; full adds Eq.~\ref{eq:intervention};
the two controls replace it by the objectives in
App.~\ref{app:intervention-theory}. Coefficient selection receives
the same validation budget, and search/decoding conventions are fixed.
All variants use the same evaluation blocks.

\paragraph{Metric aggregation.}
BlockAcc and mean worst $m$ use all $600$ test blocks. A block with any missing or invalid required prediction is scored as incorrect and assigned the maximal geometric error of $1$. For all methods, the profile diagnostic $E_{\rm shape}$ uses the same intervention contrast $U_v=H_v-2\delta_0\mathbf1$ from Eq.~\ref{eq:profile-evaluation}.

A profile is eligible when $U_v$ is finite and $\|U_v\|_2>\tau_U$, with $\tau_U=0.05$ fixed before evaluation. Each block--camera pair contributes one profile containing its $K$ intervention responses. Thus, the $600$ test blocks contribute $N_{\rm prof}=\sum_{B\in\mathcal B_{\rm test}}V_B$ profiles in total. Eligible coverage is computed for each method over this common total, and the shared eligible set contains $1{,}022$ profiles. Tab.~\ref{tab:intervention-mechanism} evaluates both $E_{\rm shape}$ and the fraction with $\eta_v^\star=0$ on this shared set, ensuring that the profile comparisons use identical observations. All blocks remain in the accuracy and geometric-error metrics. Scene-level paired resampling retains all states, cameras, and blocks from each sampled scene.

\paragraph{Transfer and comparison controls.}
Training uses 3D-FRONT, MSMU, and OSD. We check overlap across training stages and evaluation using scene/configuration identifiers and image hashes. Comparisons document the images, regions, geometric cues, prompts, decoding, and parsing used by each model. Tab.~\ref{tab:native} reports both restricted-input and native-input results for baselines that use additional visual cues. Kimi-K2.6 is an open-weight mixture-of-experts model with 1T total parameters and 32B activated parameters per token~\citep{kimi2026modelcard}; the main table reports total parameters.

\paragraph{Reward-estimator control.}
The reward estimator predicts the serialized geometric margin from image and relation-pair features. It is fitted on training scenes with a disjoint validation split, and its outputs are clipped to $[0,1]$. The text-only control uses the two serialized relations. Tab.~\ref{tab:rm} reports margin MAE, rank correlation, and fitting and scoring costs against the serialized geometric target. Direct computation evaluates the target formula from scene-state labels and requires no estimator fitting. The input bundles differ across methods, so these results quantify target approximation and compute under the listed inputs.

\begin{table}[t]
\centering\small
\caption{Intervention-block statistics. Training draws $3{,}752$ blocks with replacement over $938$ updates, visiting $1{,}447$ of the $1{,}600$ training blocks. The five splits use disjoint scene sets. Strata correspond to the magnitude--camera combinations in Tab.~\ref{tab:block-roster}.}

\label{tab:block-ledger}

\setlength{\tabcolsep}{4pt}

\begin{tabular}{llrrrrrr}
\toprule
& & \multicolumn{2}{c}{Blocks} & \multicolumn{2}{c}{Min/median/max}
& & \\
\cmidrule(lr){3-4}\cmidrule(lr){5-6}
Split & Role & Distinct & Covered & $K$ & $V_B$ & Scenes & Strata \\
\midrule
Train & GaugeDPO updates & 1600 & 1447 & 2/3/4 & 2/2/2 & 800 & 16 \\
Val & coefficient selection & 200 & 200 & 2/3/4 & 2/2/2 & 100 & 16 \\
\midrule
Test & familiar & 200 & 200 & 2/3/4 & 2/2/3 & 100 & 16 \\
Test & unseen magnitudes & 200 & 200 & 2/3/4 & 2/2/3 & 100 & 8 \\
Test & unseen combinations & 200 & 200 & 2/3/4 & 2/2/3 & 100 & 4 \\
\bottomrule
\end{tabular}

\end{table}

\begin{table}[t]
\centering\small
\caption{\textbf{Distance factors and camera configurations.} Distance levels are reported as the normalized component $\drel=\min(|\log\lambda|,\kappa)/\kappa$, where $\kappa=\log5$ and $\lambda$ is the centroid-distance ratio. The joint margin $m_k$ uses Eq.~\ref{eq:margin} with the default $\alpha=0.5$; for a distance-only change, $m_k=(1-\alpha)\drel=0.5\drel$. NN denotes the distance to the nearest training level in these normalized distance units. Camera bands specify azimuth separation. Held-out combinations retain both constituent factors in training.}

\label{tab:block-roster}

\setlength{\tabcolsep}{5pt}

\begin{tabular}{llrrl}
\toprule
Factor & Level & $\lambda$ & $\drel$ & Split role \\
\midrule
Magnitude & M1 & 1.25 & 0.1386 & train, familiar test \\
Magnitude & M2 & 1.50 & 0.2519 & train, familiar test \\
Magnitude & M3 & 2.20 & 0.4899 & train, familiar test \\
Magnitude & M4 & 3.60 & 0.7959 & train, familiar test \\
Magnitude & M5 & 4.50 & 0.9345 & train, familiar test \\
\midrule
Magnitude & H1 & 1.80 & 0.3652 & held out; NN $0.1133$ \\
Magnitude & H2 & 2.80 & 0.6397 & held out; NN $0.1498$ \\
\midrule
Camera & V1 & \multicolumn{2}{c}{$[30^\circ,60^\circ)$} & train and test \\
Camera & V2 & \multicolumn{2}{c}{$[60^\circ,90^\circ)$} & train and test \\
Camera & V3 & \multicolumn{2}{c}{$[90^\circ,120^\circ)$} & train and test \\
Camera & V4 & \multicolumn{2}{c}{$[120^\circ,150^\circ]$} & train and test \\
\midrule
Combination & (M1,V4) & \multicolumn{2}{c}{---} & held out, both factors seen \\
Combination & (M2,V3) & \multicolumn{2}{c}{---} & held out, both factors seen \\
Combination & (M4,V1) & \multicolumn{2}{c}{---} & held out, both factors seen \\
Combination & (M5,V2) & \multicolumn{2}{c}{---} & held out, both factors seen \\
\bottomrule
\end{tabular}

\end{table}

\section{Additional Experimental Results}
\label{app:exp-results}

Supplementary experiments assess transfer to additional spatial and embodied benchmarks, sensitivity to loss weights, robustness to input changes and training seeds, and the effect of baseline inputs and reward estimation. Base, SFT, and GaugeVLM checkpoints use the same evaluation pipeline; comparisons against SFT measure the gains from the complete preference stage. Matched objective controls separate the effects of individual supervision terms.

\subsection{Additional evaluation}
\providecommand{\gain}[1]{\textbf{+#1}}
\providecommand{\drop}[1]{\textcolor{gray}{$-$#1}}

\begin{table}[t]
\centering
\small
\setlength{\tabcolsep}{5pt}
\renewcommand{\arraystretch}{1.0}
\caption{\textbf{Additional metric and spatial benchmark results.} Task definitions differ across benchmarks; driving results appear in the main table.}
\label{tab:metricgroup}
\begin{tabular}{lllc}
\toprule
Benchmark & Domain & \method{} main metric & $\Delta_{\text{vs SFT}}$ \\
\midrule
SPAR-Bench            & indoor room (ScanNet)  & overall 34.0           & \gain{8.7} \\
SPAR-Bench (mv)       & robotic multi-view     & dist\_oo 38.3          & \gain{16.5} \\
QSpatial$^+$          & indoor/mixed           & 64.4                   & \gain{18.9} \\
Omni3DBench           & mixed                  & MC 60.5                & \gain{5.8} \\
SPBench-SI            & indoor desktop         & size 49.8              & \gain{5.4} \\
\bottomrule
\end{tabular}
\end{table}

\begin{table}[t]
\centering
\small
\setlength{\tabcolsep}{4pt}
\renewcommand{\arraystretch}{1.0}
\caption{\textbf{Additional embodied benchmark results.}}
\label{tab:downstream}
\begin{tabular}{lllcccc}
\toprule
Task (bench) & Domain & sub-metric & base & SFT & \method{} & $\Delta_{\text{vs SFT}}$ \\
\midrule
SPAR (mv)    & embodied & dist\_oc (obj-cam)    & 15.7 & 27.1 & \textbf{41.2} & \gain{14.1} \\
SPAR (mv)    & embodied & dist\_oo (obj-obj)    & 19.8 & 21.8 & \textbf{38.3} & \gain{16.5} \\
RoboSpatial  & embodied & configuration         & 69.1 & 74.8 & \textbf{78.1} & \gain{3.3} \\
RefSpatial   & embodied & overall               & 13.0 & 16.0 & \textbf{17.0} & \gain{1.0} \\
\bottomrule
\end{tabular}
\end{table}

\paragraph{Additional spatial benchmarks.}
Tab.~\ref{tab:metricgroup} extends the evaluation to
SPAR-Bench~\citep{zhang2025from},
QSpatial$^+$~\citep{liaoReasoningPathsReference2024},
Omni3D-Bench~\citep{marsili2025visual}, and
SPBench-SI~\citep{liSpatialLadderProgressiveTraining2025},
covering spatial perception and reasoning, quantitative distance
estimation, and object-size estimation.
\method{} improves over SFT on every reported metric, including
gains of $8.7$ points on SPAR-Bench overall, $18.9$ points on
QSpatial$^+$, and $5.4$ points on SPBench-SI size estimation.
On the multiple-choice component of Omni3D-Bench (MC), it reaches
$60.5\%$, a gain of $5.8$ percentage points.
The QSpatial$^+$ result is repeated from the main table for reference.
Since task definitions and scoring rules differ across benchmarks,
we report their individual metrics rather than aggregate them
into a single score.

\paragraph{Spatial capabilities relevant to embodied tasks.}
Tab.~\ref{tab:downstream} further examines multi-view distance
estimation on SPAR-Bench~\citep{zhang2025from}, spatial configuration
understanding on RoboSpatial~\citep{song2025robospatial}, and spatial
referring on RefSpatial-Bench~\citep{zhou2025roborefer}.
Here, \texttt{dist\_oc} and \texttt{dist\_oo} denote object--camera
and object--object distance estimation, respectively.
Relative to SFT, \method{} improves these two SPAR-Bench metrics
by $14.1$ and $16.5$ points, while RoboSpatial configuration accuracy
increases from $74.8\%$ to $78.1\%$ and the RefSpatial-Bench overall
score rises from $16.0$ to $17.0$.
These results support transfer to spatial capabilities relevant
to embodied interaction, with a smaller improvement on spatial
referring.
The SPAR-Bench \texttt{dist\_oo} result is shared with
Tab.~\ref{tab:metricgroup} and represents the same evaluation.

\subsection{Camera Outcomes and Shortcut Controls}
\label{app:view-details}
\label{app:controls-full}
\label{sec:robustness}

The camera evaluation uses $100$ scenes, two object configurations per scene, and two camera views per configuration, yielding $200$ groups and $400$ required views. Each group shares its anchor frame and true clock--distance relation. Under the tolerances in App.~\ref{app:view-audit}, GaugeVLM achieves $69.5\%$ group accuracy and $63.0\%$ accurate agreement, compared with $50.0\%$ and $40.0\%$ for SFT (Tab.~\ref{tab:crossview}). Incorrect agreement decreases from $20.0\%$ to $11.5\%$. These results show that improved agreement is accompanied by improved correctness. All models use the same evaluation groups and denominator, with the scoring rules in App.~\ref{app:scoring}. The $6$ incomplete groups contribute $3.0\%$ to the invalid rate for every model. Additional parsing failures account for $4$ groups in SFT, $2$ in the pair-only ablation, and $1$ in each single-weight ablation. GaugeVLM produces valid outputs for all complete groups, giving an invalid rate of $3.0\%$.

\Needspace{5\baselineskip}
\paragraph{Shortcut controls.}
The controls below separate image dependence, option ordering and
negative-construction cues. They complement the matched objective
definitions in App.~\ref{app:intervention-protocol}.

Tab.~\ref{tab:visual} evaluates the same checkpoint under changes to image availability, answer-option order, and state presentation. PairAcc. uses $205$ relocation pairs, and accurate agreement uses the $200$ camera groups. Permuting options or reversing state order changes PairAcc. by less than one percentage point. Masking the image or using text alone reduces PairAcc. from $66.3\%$ to $25.9\%$ and $26.8\%$, respectively, indicating reliance on visual input. Predictions are mapped back to semantic labels before scoring.

Tab.~\ref{tab:negative} separates the effects of reciprocal/sign balancing and coherent reasoning steps in preference construction. Text-only preference accuracy is evaluated on $200$ held-out pairs, with randomized candidate order and scene-disjoint fitting and validation. Combining balancing and coherence reduces text-only preference accuracy from $88.0\%$ to $54.5\%$, while increasing accurate camera agreement from $54.0\%$ to $63.0\%$. The combined construction therefore reduces detectable textual shortcuts while improving spatial performance.

\begin{table}[t]
\centering\small
\caption{\textbf{Input and presentation controls.} Intervention PairAcc. and accurate agreement use separate evaluation sets.}
\label{tab:visual}

\begin{tabular}{lrr}
\toprule
Condition & Intervention PairAcc & Accurate agreement \\
\midrule
Default options and state order & 66.3 & 63.0 \\
Permuted answer options & 65.9 & 62.5 \\
Reversed state presentation & 65.4 & 62.5 \\
Image masked & 25.9 & 18.5 \\
Text only & 26.8 & 19.5 \\
\bottomrule
\end{tabular}
\end{table}
\begin{table}[t]
\centering\small
\caption{\textbf{Construction controls.} Text-only preference accuracy measures discrimination of chosen/rejected pairs, not spatial answer accuracy.}
\label{tab:negative}

\begin{tabular}{lrrrr}
\toprule
Construction & MSMU & QSpat & Acc. agree. & Text-only pref. acc. \\
\midrule
One-sided, inconsistent rationale & 60.0 & 62.4 & 54.0 & 88.0 \\
Reciprocal/sign-balanced only & 62.5 & 63.4 & 56.5 & 79.0 \\
Coherent rationale only & 60.0 & 63.4 & 59.0 & 66.5 \\
Balanced and coherent & 62.5 & 64.4 & 63.0 & 54.5 \\
\bottomrule
\end{tabular}
\end{table}

\subsection{Training-Seed Robustness and Paired Uncertainty}
\label{app:seed-robustness}
\label{app:seeds}

\paragraph{Matched preference-stage runs.}
We use five random seeds, $\{11,23,37,53,71\}$, for each of
seven Qwen2.5-VL-7B objectives in Tab.~\ref{tab:seed-mechanism}.
All 35 runs start from the same SFT checkpoint and frozen reference.
Source scenes, pair and block pools, train/validation/test splits,
hyperparameters, and evaluation manifests are fixed across seeds.
Each run uses 30,000 pair draws, 938 optimizer updates, and four block
draws per update. We evaluate the final scheduled checkpoint on the test
set using the same checkpoint-selection rule for all runs.

Across seeds, we vary the training RNG streams governing pair/block
sampling, data order, and any enabled stochastic model operations.
Within a seed, all objectives share the same serialized pair/block
draw schedule and aligned stochastic-operation streams wherever applicable.
Separate RNG streams handle objective-specific operations, including
the group-preserving shuffled-margin permutation. Removing a loss
does not change the data exposure of the remaining branches.
Zero-, constant-, and shuffled-offset pair losses change only the pair term;
their direct and interaction terms remain active. Pair + direct sets
$\lambda_{\rm int}=0$, while the two profile controls replace only
the interaction objective. All settings are frozen before the five-seed sweep.

\begin{table*}[t]
\centering
\caption{\textbf{Preference-stage robustness.}
Mean $\pm$ sample SD computed from five accuracy scores
per cell. Accuracy values are percentages.}
\label{tab:seed-mechanism}
\begingroup
\footnotesize
\setlength{\tabcolsep}{4pt}
\renewcommand{\arraystretch}{1.10}
\begin{tabular*}{\textwidth}{@{\extracolsep{\fill}}lccccc@{}}
\toprule
Objective
& \shortstack{MSMU\\distance}
& QSpatial$^{+}$
& PairAcc
& GroupAcc$_{\rm tol}$
& BlockAcc \\
\midrule
Zero-offset pair loss
& $49.5\pm2.1$ & $49.1\pm1.7$ & $56.2\pm1.5$
& $60.2\pm1.4$ & $55.6\pm1.5$ \\

Constant pair offset
& $52.0\pm1.1$ & $52.1\pm1.5$ & $57.8\pm1.3$
& $62.0\pm1.3$ & $57.3\pm1.4$ \\

Shuffled pair offset
& $52.0\pm2.1$ & $51.3\pm1.8$ & $57.4\pm1.6$
& $61.5\pm1.5$ & $56.8\pm1.6$ \\

Pair + direct
& $57.5\pm1.8$ & $61.0\pm2.8$ & $62.0\pm1.2$
& $67.2\pm1.2$ & $59.7\pm1.3$ \\

\quad + Single-gap profile
& $59.0\pm1.4$ & $60.8\pm1.1$ & $63.3\pm1.1$
& $68.0\pm1.1$ & $63.4\pm1.2$ \\

\quad + EqSim-style symmetry
& $59.5\pm1.1$ & $61.6\pm1.3$ & $63.9\pm1.2$
& $68.3\pm1.2$ & $64.7\pm1.2$ \\

\midrule
GaugeDPO (full)
& $62.0\pm1.1$ & $64.0\pm1.1$ & $66.0\pm1.0$
& $69.3\pm1.0$ & $69.6\pm1.1$ \\
\bottomrule
\end{tabular*}
\endgroup
\end{table*}
\paragraph{Evaluation units and retained predictions.}
Every run uses identical prompts, preprocessing, deterministic decoding,
canonical candidate search, parsing, and invalid-output rules.
Native benchmark generation and canonical sequence-MAP evaluation
remain distinct protocols. The intervention test contains 600 blocks
from 300 disjoint scenes: each of three settings contains 100 scenes
and two blocks per scene. The camera audit contains 200 groups from
100 scenes. The relocation test contains 205 pairs; all pairs from
the same source scene remain one statistical cluster.
Public benchmarks are clustered by source scene when scene identifiers
are available, and otherwise by source image.
We retain every run's checkpoint hash, configuration, sampled-unit IDs,
predictions, parse flags, and per-unit scores, including failures.

\paragraph{Training variation and paired differences.}
For each method, we report the mean and sample SD of its five
seed-level scores on the unchanged test manifest. For a matched
comparison, let $d_s=M_{{\rm full},s}-M_{{\rm control},s}$.
The seed-only 95\% interval is
$\bar d\pm t_{0.975,4}\,s_d/\sqrt{5}$; it quantifies training
variation conditional on this test set and, for
Tab.~\ref{tab:seed-mechanism}, this SFT checkpoint.
We also report the number of seeds with $d_s>0$; ties are not wins.

\paragraph{Uncertainty across seeds and test scenes.}
We additionally use 10,000 paired bootstrap replicates with analysis
seed 2027. Each replicate independently resamples the five seed
indices and the test-scene clusters with replacement. The same
resampled seed indices and scene multiset are used for both methods;
the scene multiset is also shared across all sampled seeds.
All questions, pairs, configurations, views, and blocks belonging to
a selected scene are retained together. For the block test, scene
resampling is stratified by the three test settings, preserving their
equal weights. We recompute each metric from its original numerator
and denominator, then average the paired differences over sampled
seeds. The 2.5th and 97.5th percentiles give the seed-and-scene interval.
This separates training SD from uncertainty in the method difference
and avoids counting repeated predictions on the same scene as
independent test observations.
Tab.~\ref{tab:seed-paired-gains} shows positive gains for all reported comparisons across all five seeds. The joint seed-and-scene intervals also remain above zero, supporting robustness to both training variation and the sampled test scenes.

\begin{table*}[t]
\centering
\caption{\textbf{Paired gains over controls.}
Full GaugeDPO minus each control, in percentage points. Primary
comparisons are specified before the repeated runs. Both confidence
intervals describe the paired difference.}
\label{tab:seed-paired-gains}
\begingroup
\footnotesize
\setlength{\tabcolsep}{4pt}
\renewcommand{\arraystretch}{1.10}
\begin{tabular*}{\textwidth}{@{\extracolsep{\fill}}llcccc@{}}
\toprule
Control & Metric & Mean $\Delta$
& \shortstack{95\% CI\\seeds only}
& \shortstack{95\% CI\\seeds + scenes}
& \shortstack{Positive\\seeds / 5} \\
\midrule
\multicolumn{6}{@{}l}{\textit{Primary comparisons}} \\
Shuffled pair offset & QSpatial$^{+}$ & $+12.7$ & $[10.7,\,14.7]$ & $[8.8,\,16.4]$ & 5 / 5 \\
Pair + direct   & BlockAcc      & $+9.9$ & $[8.3,\,11.5]$ & $[6.9,\,12.8]$ & 5 / 5 \\
\midrule
\multicolumn{6}{@{}l}{\textit{Secondary comparisons}} \\
Zero-offset pair loss & QSpatial$^{+}$ & $+14.9$ & $[12.9,\,16.8]$ & $[11.1,\,18.5]$ & 5 / 5 \\
Constant pair offset & QSpatial$^{+}$ & $+11.9$ & $[9.9,\,13.8]$ & $[8.3,\,15.4]$ & 5 / 5 \\
Single-gap      & BlockAcc      & $+6.2$ & $[4.7,\,7.7]$ & $[3.5,\,8.9]$ & 5 / 5 \\
EqSim-style     & BlockAcc      & $+4.9$ & $[3.2,\,6.6]$ & $[2.1,\,7.6]$ & 5 / 5 \\
Pair + direct   & Unseen-mag. BlockAcc & $+10.6$ & $[8.2,\,13.0]$ & $[6.2,\,15.1]$ & 5 / 5 \\
Pair + direct   & Unseen-comb. BlockAcc & $+10.4$ & $[7.7,\,13.1]$ & $[5.4,\,15.3]$ & 5 / 5 \\
\bottomrule
\end{tabular*}
\endgroup
\end{table*}

\paragraph{Independent two-stage replications.}
To assess sensitivity beyond a fixed SFT initialization, we separately
repeat the complete SFT--preference pipeline five times for each
backbone. Each repetition independently trains SFT from the released
base model, then branches into matched zero-offset pair-loss and full-GaugeDPO
preference runs, sharing that repetition's SFT reference and draw
schedules. Stage-specific RNG streams are derived from the repetition
seed. SFT and preference update budgets and the final-checkpoint rule
are fixed across repetitions. This adds five SFT runs and ten
preference runs per backbone: 15 SFT and 30 preference runs across
three backbones. Tab.~\ref{tab:seed-end-to-end} shows consistent gains across all three backbones when variation in both training stages is included.

\begin{table*}[t]
\centering
\caption{\textbf{Robustness of the complete training pipeline.}
Five independent SFT--preference repetitions per backbone.
Entries are mean $\pm$ sample SD in percent. SFT checkpoints are
shared by the two preference objectives within each repetition.}
\label{tab:seed-end-to-end}
\begingroup
\footnotesize
\setlength{\tabcolsep}{4pt}
\renewcommand{\arraystretch}{1.10}
\begin{tabular*}{\textwidth}{@{\extracolsep{\fill}}llcccc@{}}
\toprule
Backbone & Training
& \shortstack{MSMU\\distance}
& QSpatial$^{+}$ & PairAcc & BlockAcc \\
\midrule

Qwen2.5-VL-7B
& SFT
& $47.0\pm2.1$
& $45.3\pm1.9$
& $54.2\pm1.8$
& $45.2\pm2.0$ \\

& Zero-offset pair loss
& $49.5\pm2.1$
& $49.1\pm1.8$
& $56.3\pm1.7$
& $55.4\pm1.9$ \\

& GaugeDPO
& $62.0\pm2.1$
& $63.8\pm1.5$
& $65.8\pm1.4$
& $69.4\pm1.5$ \\

\midrule

GLM-4.1V-9B
& SFT
& $67.0\pm2.1$
& $49.1\pm2.1$
& $67.0\pm1.9$
& $47.8\pm2.1$ \\

& Zero-offset pair loss
& $68.5\pm2.2$
& $52.5\pm1.9$
& $67.8\pm1.7$
& $56.9\pm1.9$ \\

& GaugeDPO
& $72.0\pm2.1$
& $58.0\pm1.7$
& $69.5\pm1.5$
& $71.2\pm1.6$ \\

\midrule

Pixtral-12B
& SFT
& $49.5\pm2.7$
& $7.7\pm1.5$
& $31.8\pm2.1$
& $38.4\pm2.3$ \\

& Zero-offset pair loss
& $51.0\pm2.2$
& $9.3\pm1.5$
& $35.6\pm2.0$
& $45.1\pm2.1$ \\

& GaugeDPO
& $52.0\pm2.1$
& $11.7\pm1.5$
& $40.6\pm1.8$
& $52.7\pm1.9$ \\

\bottomrule
\end{tabular*}
\endgroup
\end{table*}

\subsection{Baseline Inputs and Reward-Estimator Comparisons}
\label{app:comparisons}
\label{app:completion}
These comparisons assess the effects of additional baseline inputs and the accuracy and cost of estimating the geometric target. Their settings are specified in App.~\ref{app:intervention-protocol}.

\paragraph{Restricted and native inputs.}
Tab.~\ref{tab:native} compares restricted inputs with the additional visual cues supported by each baseline. Native cues improve MSMU distance accuracy from $25.0\%$ to $30.0\%$ for SpatialReasoner and to $32.5\%$ for SpatialRGPT. GaugeVLM achieves $62.5\%$ on MSMU distance and $64.4\%$ on QSpatial$^{+}$, exceeding both native-input baselines on these two metrics under the listed configurations.

\begin{table}[t]
\centering\small
\caption{Restricted-input and native-input comparison. Each native configuration retains the model's specified cue bundle.}
\label{tab:native}

\setlength{\tabcolsep}{6pt}

\begin{tabular}{llrr}
\toprule
Model & Input mode & MSMU dist. & QSpatial$^{+}$ \\
\midrule
SpatialReasoner & restricted & 25.0 & 49.5 \\
SpatialReasoner & native cue bundle & 30.0 & 53.5 \\
SpatialRGPT & restricted & 25.0 & 55.4 \\
SpatialRGPT & native cue bundle & 32.5 & 60.4 \\
GaugeVLM & standard input & 62.5 & 64.4 \\
\bottomrule
\end{tabular}

\end{table}

\paragraph{Approximation error and computational cost.}
The image-conditioned regressor in Tab.~\ref{tab:rm} achieves $0.087$ margin MAE and $0.72$ rank correlation against the serialized geometric target. Direct computation requires no estimator fitting and $0.02$ GPU-hours for scoring in the reported setting. Its zero MAE and unit rank correlation express algebraic agreement with the target formula; they do not measure error against continuous physical geometry. The table therefore complements the downstream margin ablation by reporting the cost and approximation error of the target-construction path under each method's stated inputs.

\begin{table}[t]
\centering\small
\caption{Reward-estimator comparison against serialized geometric targets. The direct-margin row gives algebraic agreement with the target formula.}
\label{tab:rm}

\setlength{\tabcolsep}{4pt}

\begin{tabular}{llrrrr}
\toprule
Estimator & Inputs & \shortstack{Margin\\MAE} & Rank $\rho$ & \shortstack{Fit\\GPU-h} & \shortstack{Score\\GPU-h} \\
\midrule
Scalar regressor & image + two relations & .087 & .72 & 1.2 & .3 \\
Text-only regressor & two serialized relations & .104 & .65 & .5 & .1 \\
Direct geometric margin & scene-state labels & 0.000 & 1.00 & 0.0 & .02 \\
\bottomrule
\end{tabular}

\end{table}

\end{document}